\documentclass[10pt,twocolumn,twoside]{IEEEtran}
\usepackage[comma,numbers,square,sort&compress]{natbib}
\usepackage{textcomp}
\usepackage{amsmath}
\usepackage{amssymb}
\usepackage{float}
\usepackage{amsthm}
\usepackage{graphicx}
\usepackage{subfigure}
\usepackage{epstopdf}
\usepackage{color}
\usepackage{verbatim}
\usepackage{tikz,times}
\usepackage{algorithm}
\usepackage{algorithmic}

        \DeclareMathOperator{\avg}{avg}
                \DeclareMathOperator{\net}{net}
    \DeclareMathOperator{\blockdiag}{blockdiag}
    \DeclareMathOperator{\diag}{diag}

    \newcommand{\norm}[1]{\left\lVert#1\right\rVert}

    \newtheorem{theorem}{Theorem}
    
    \newtheorem{lemma}{\textbf{Lemma}}
    \newtheorem{corollary}{\textbf{Corollary}}

    \newtheorem{definition}{\textbf{Definition}}
    \newtheorem{assumption}{\textbf{Assumption}}

    \newtheorem{remark}{Remark}

\begin{document}
\title{Navigating A Mobile Robot Using Switching   Distributed Sensor Networks}
\author{Xingkang He, Ehsan Hashemi, Karl H. Johansson
\thanks{ The work is supported by the Knut \& Alice Wallenberg Foundation and the Swedish Research Council. }
\thanks{X. He and K. H. Johansson are with Division of Decision and Control Systems, School of Electrical Engineering and Computer Science. KTH Royal Institute of Technology, Sweden ((xingkang,kallej)@kth.se).}%
\thanks{E. Hashemi is with 
Department of Mechanical Engineering, University of Alberta,   Canada (e-mail: ehashemi@ualberta.ca).}
}

\maketitle

\begin{abstract}
This paper proposes a method to navigate   a mobile robot  by estimating its state over a number of distributed sensor networks (DSNs) such that it can successively accomplish a sequence of tasks, i.e.,  its state  enters each targeted set and stays inside no less than the desired time, under a resource-aware, time-efficient, and computation- and communication-constrained setting.
We propose a new robot state estimation and navigation architecture, which integrates
an event-triggered task-switching feedback controller for the robot and a two-time-scale distributed state estimator for each  sensor. The architecture has  three major advantages over  existing approaches: First, in each task  only one DSN is active  for sensing and estimating the robot state, and for different tasks the robot can switch the active DSN by taking resource saving and system performance into account; Second,  the robot only needs to communicate with one active sensor at each time to obtain its state information from the active DSN;  Third, no online optimization is required. With the controller, the robot is able to accomplish a task by following a reference trajectory and switch to the next task when an event-triggered condition is fulfilled.  
With the estimator, each active sensor is able to estimate the robot state. 
Under proper conditions, we prove that the state estimation error and the trajectory tracking deviation are upper bounded by two time-varying sequences respectively, which play an essential role in the event-triggered condition. Furthermore, we find a sufficient condition for accomplishing a task and  provide an upper bound of running time for the task.  Numerical simulations of an indoor robot's localization and navigation are provided to validate the proposed architecture.

%
%
%The switching of sensor state (active or inactive) is controlled by the robot via intermittent communications.  
%
%control a robot in an uncertain environment through a connected sensor network,  such that the robot is able to accomplish a sequence of tasks, namely, reaching certain  sets  in order.    
%Under constrained communication between the robot and the sensor network, an event-triggered  task-switching control architecture is established, such that   the robot state remains in each task set  for the desired time and then switches to the next task.  To  estimate the robot state, a two time-scale distributed estimator  is proposed for each sensor by employing a local predicted control.
%Under mildest system conditions (i.e., stabilization and collective detectability), 
%we provide explicit design criteria for the filter  parameters, such that the state estimation error  is asymptotically upper bounded.
%Moreover, we prove that the tracking deviation between the robot state and a reference trajectory is upper bounded. When the system is noise-free, we show that the estimation error of each sensor and tracking deviation of the robot both tend to zero.
%Furthermore,  we obtain an upper bound of the running time of  each task, i.e., the transition time from starting one  task to another. To reduce the task running time, the start control time is designed by solving an offline optimization problem.
%Numerical simulations are provided to validate the proposed architecture and the  developed results.	
\end{abstract}

\begin{IEEEkeywords}
Navigation, distributed estimation,  event-triggered, task-switching, distributed sensor network,	  trajectory tracking
%Enter key words or phrases in alphabetical 
%order, separated by commas. For a list of suggested keywords, send a blank 
%e-mail to keywords@ieee.org or visit \underline
%{http://www.ieee.org/organizations/pubs/ani\_prod/keywrd98.txt}
\end{IEEEkeywords}

\section{Introduction}
\label{sec:introduction}
Autonomous mobile robots with augmented state estimation and navigation systems over sensor networks are revolutionizing accurate navigation and controls for indoor and outdoor applications, such as service robots in dynamic environments, automated storage/retrieval with mobile robots in a warehouse \cite{yang2008multi, lenain2010mixed, atia2015integrated, chung2015indoor}. In addition, in cooperative intelligent transportation systems, reliable navigation through vehicle-to-infrastructure and over sensor networks, plays a vital role in reliable decision-making, safe motion planning, and controls for automated driving in urban settings \cite{li2020toward, kuutti2018survey, pare2019networked}.

Mobile robot navigation has been intensively studied through centralized frameworks \cite{flynn1988combining,kam1997sensor,desouza2002vision}, where a control center (e.g., the robot in autonomous navigation) collects the data of all sensors to estimate the robot state and then design feedback control signals such that the robot is able to reach targeted areas by following reference trajectories. Moreover, the center  sends the estimates and control signals to each sensor for further monitoring. Such timely information synchronization between the sensors and the  center 
inevitably increases bandwidth issues,  channel burden, and  energy consumption, thus these approaches are not effective  when the number of the deployed sensors is large. Moreover, although some approaches  (e.g., optimization-based-feedback temporal logic control \cite{belta2019formal} and the references therein) provide high-efficient navigation solutions by designing optimized feedback controllers, they  necessitate high online computation capability of the control center, which however is infeasible  when  the center has limited online computation capability and  is  inefficient in facing a large number of data. 
% Thus,  This limit their applications to the cases with constrained communication or computation capability.

%However, these solutions are usually not applicable to the scenarios where the 
% the communication between sensors and the control center is constrained or the control center has limited online computation capability. It becomes more challenging when the number of deployed sensors is very large.
% For example,  
% formal methods provide effective solutions for robot navigation by  designing feedback controllers. However, they  necessitate high online computation capability of the control center and \cite{belta2019formal}. This limit their applications to the cases with constrained communication or computation capability.
% 
 
In contrast to the well-developed centralized estimation frameworks (e.g., \cite{yan2017networked,xu2020reset} and the references therein), state estimation over distributed sensor networks (DSNs), where each sensor only communicates with its neighboring sensors, shows advantages in structure robustness and parallel data processing, and thus draws more and more attention in recent years. An asynchronous distributed localization algorithm is proposed in \cite{huang2018asynchronous}  over DSNs with time-varying delays.  
For linear stochastic systems, distributed Kalman-type filters have been proposed (see \cite{battistelli2014kullback, yang2017stochastic,he2020distributed}  and references therein). For linear time-invariant   deterministic systems,  
distributed state observers  have been studied \cite{wang2017distributed, mitra2018distributed}. 
Although distributed estimation problems have been well investigated in the literature, few results are given on the design of distributed estimators based feedback controllers.  
%A uniﬁed framework for distributed Kalman filtering and control, in sensor networks, is developed in \cite{talebi2019distributed}, in which the sensors mirror the process of the centralized Kalman filters in a distributed scheme through embedded average consensus fusion of local state estimates and covariances. 
Since   timely control signal synchronization between    control center and   DSNs induces a large amount of communications, new architectures able to alleviate   communications between the center and DSNs are expected.  Moreover,   the  number of sensors in a DSN can lead to a tradeoff between   resource consumption and estimation performance.  On one hand,  running a large-size DSN with many redundant sensors   leads to a wast of resources, especially  when the desired estimation performance can be   ensured with a  small-size DSN.  On the other hand, a too small-size DSN may   provide insufficient information for the desired estimation performance.  Thus, in different scenarios, 
how to timely choose proper DSNs balancing     resource consumption and   estimation performance needs investigation. Besides, how to coordinate the switching  of DSNs is of interest.

%
% In order to estimate the states for a cooperative robot network, a distributed observer is designed in \cite{marino2017distributed} where an estimator-based adaptive local control strategy is designed to address model uncertainties. A consensus filter is used to reconfigure mobile sensors in \cite{freundlich2017distributed} in order to design a control law robust to state disagreement errors and to estimate a set of hidden states.  
% 

In multi-task robot navigation scenarios, such as   multiple trajectory tracking or multiple area coverage, it is certainly desired to timely execute the next task upon accomplishing the current task. This real-time task adjustment is also essential in localization and decision-making over complex environments \cite{galceran2013survey, lin2017distributed}. 
Event-triggered strategies have been studied with a favorable tradeoff between resource consumption and system performance.  
% A Kalman-based consensus with time-triggered transmission mechanism, based on the send-on-delta data conveyance scheme, is designed in \cite{li2016event} for distributed estimation over sensor networks.
% In order to maintain the Gaussianity of the system, 
 A stochastic event-triggered sensor schedule  is presented in \cite{weerakkody2015multi} for state estimation in multi-sensor scenarios. 
% A consensus Kalman filter algorithm is developed in \cite{battistelli2018distributed} by enforcing each node to transmit its local data to the neighbors when it is considered as significant for estimation which is it notably deviates from the information predicted from the last transmitted one. 
 In \cite{xu2019event},   an event-based robust distributed synchronization is proposed for systems with disturbance uncertainties in order to lower the update rate/frequency.  Interested readers can refer to \cite{peng2018survey} for a literature survey of event-triggered estimation and control. In contrast, event-triggered multi-task switching is to provide   switching decision when some certainty level is reached, such as switching to the next task upon the accomplishment of the current task.  The certainty level is usually built on real-time performance evaluation, such as error bounds. However,  for multi-task navigation under uncertainties, the error bounds  may not be deterministic due to random switching time instants, such that the existing design and analysis approaches do not work well. 

\subsection*{Contributions}
%In \cite{bucsoniu2020learning}, a data-driven component to estimate the rate function for transmitting the data over a sensor network is provided, and an optimal controller to guide the robot given the current rate function estimates, is designed

This paper studies how to navigate a mobile robot   by estimating its state over a number of DSNs 
such that the robot is able to  successively accomplish a sequence of tasks, i.e.,  its state  enters each targeted set and stays inside no less than the desired time, in a resource-aware, time-efficient, and computation- and communication-constrained setting. 
To reduce energy consumption of running redundant sensors, the robot is  able to choose which DSN is active in state sensing and estimation for each task.
%The problem is built in a novel \textit{constrained communication} scheme in two directions: When a  task is accomplished, the robot broadcasts switching signals of task and sensor state to the sensor network, such that a group of .
% each active sensor can switch its estimation to the next task. 
%to determine the  state of sensors and switch to the next task
%broadcasts a switching signal to the sensor network when a certain task is accomplished; only one sensor  communicates with the robot at each time.
The  contributions of the paper are three-fold:
\begin{enumerate}
	
	\item We propose an integrated state estimation and navigation architecture (Fig.~\ref{fig:diagram}) consisting of a task-switching controller for the robot and a distributed estimator for each active sensor. In the architecture, i) the robot only broadcasts a key message to a subset of sensors upon the accomplishment of a task, and ii)  only one active sensor shares its estimates with the robot persistently in  a task. Thus, the architecture requires less communications between the robot and sensors than the centralized  schemes.

	\item  We design an event-triggered feedback controller (Algorithm~\ref{alg:B}), such that the robot is able to accomplish a task by following a reference trajectory and switch to the next task when an event-triggered condition is fulfilled.  Moreover, we propose a two-time-scale distributed state estimator (Algorithm~\ref{alg:A})  for each active sensor, such that the sensor is able to estimate the robot state by using its local measurements and the information from its  active neighbors.  Algorithms~\ref{alg:B} and \ref{alg:A} are integrated into the navigation architecture.
%	iii) We 	integrate the controller and the estimator  under constrained communication between the robot and the sensor network in the sense that the robot broadcasts key information to the network once a  task is accomplished and   receives its state information only from one active sensor at each time. 
	
%   The proposed architecture provides a promising and effective solution for the navigation and localization  of robots  monitored by sensor networks.

	\item  Under proper conditions, we prove that the state estimation error and  trajectory tracking deviation are upper bounded by two time-varying sequences respectively (Theorem~\ref{thm_error_bound}), which play an essential role in the event-triggered controller design and can be used for performance evaluation.    We find a sufficient condition for accomplishing a task and  provide an upper bound of running time for the task via solving an optimization problem (Theorem~\ref{thm_control}). 

\end{enumerate}

%Under constrained communications between the robot and the sensor network, an event-triggered  task-switching control architecture is proposed, such that   the robot state remains in each task set  for the desired time and then switches to the next task.  To  estimate the robot state over the uncertain environment, a two time-scale distributed estimator  is proposed for each sensor by employing a local predicted control.

%Under mild system conditions (i.e., stabilization and collective detectability), 
%we provide explicit design criteria for the filter  parameters, such that the state estimation error  is asymptotically upper bounded.
%Moreover, we prove that the tracking deviation between the robot state and a reference trajectory is upper bounded. When the system is noise-free, we show that the estimation error of each sensor and tracking deviation of the robot both tend to zero.
%
%Furthermore,  the running time of  each task, i.e., the transition time from starting one  task to another, is analyzed and proved to be bounded by a finite value. To reduce the task running time, the start control time is designed by solving an offline optimization problem.
%Numerical simulations are provided to validate the proposed architecture and the  theoretical results.

This paper builds on  a  conference work   presented in \cite{he2020event}, but substantial improvements have been made.   First, the active DSN has been extended from a fixed one to be a switching one. Second, the algorithms have been modified to adapt to the new setting.
Third, the results (Theorem~\ref{thm_error_bound})  are generalized  and new theoretical
results (Theorem~\ref{thm_control}) are added.   Fourth, detailed  proofs  are added, and  more
literature comparisons  and simulation results are provided.

The remainder of the paper is organized as follows. Section~\ref{sec:problem}
is on  problem formulation with a motivating example of indoor robot localization and navigation. Section~\ref{sec:architecture}  provides the design of the architecture, the controller, and the estimator.
Section~\ref{sec:theory} analyzes the main properties of the controller and the estimator.
Numerical
simulations for the motivating example are given in Section~\ref{sec:simulation}. Section~\ref{sec:conclusion} concludes   this paper.
Some proofs are given in the Appendix.

\textbf{Notations}. 
%	The superscript ``T" represents the transpose. $I_{n}$ stands for the identity matrix with $n$ rows and $n$ columns. $E\{x\}$ denotes the mathematical expectation of the stochastic variable $x$, and 
The operators $\diag\{\cdot\}$ and $\blockdiag\{\cdot\}$ means that scalar and   block matrices are arranged in diagonals, respectively. 
%	  $\mathbb{N}^+$ denotes the set of positive natural numbers. 
$\mathbb{R}^{n\times m}$ is the set of real matrices with $n$ rows and $m$ columns. $\mathbb{R}^n$ stands for the set of $n$-dimensional real vectors. $\mathbb{N}^+$ stands for the set of positive integers, and $\mathbb{N}=\mathbb{N}^+\cup 0.$
$I_{n}$ stands for the $n$-dimensional identity matrix and   subscript $n$ may  be omitted to ease the notation.
$\textbf{1}_N$ stands for the $N$-dimensional vector with all elements being one.  
For integers $m$ and $n$ with $m<n$, let $[m,n]=\{m,m+1,\dots,n\}$.
$A\otimes B$ is the Kronecker product of $A$ and $B$.  $\norm{x}$ is the 2-norm of a vector $x$. $\norm{A}$ is the induced 2-norm, i.e., $\norm{A}=\sup_{x\neq 0}\norm{Ax}/\norm{x}$.   
$\lambda_{\min}(A)$, $\lambda_2(A)$ and $\lambda_{\max}(A)$ are the minimum, second minimum and maximum eigenvalues of a real-valued symmetric matrix  $A$, respectively.
% The matrix $M_{i:j,m:n}$ stands for the block matrix of the matrix $	 M$ ranging from the rows $i$ to $j$, and the columns $m$ to $n$. $\mod(\cdot)$ and $\lfloor\cdot \rfloor$ are the modulo operation and  floor operation, respectively.
  $A\preceq B$ (resp. $A\prec B$) means $B-A$ is a positive definite matrix (resp. positive semi-definite matrix). Also, $A\preceq B$  (resp. $B\prec A$) is equivalent to $B\succeq A$ (resp. $B\succ A$).   
%  All the matrices, vectors, and scalars in this paper are real-valued.
  $|S|$ denotes the cardinality of set $S$.

\section{Problem Formulation}\label{sec:problem}
\subsection{Motivating example}\label{subsec:example}
%[width=0.45\textwidth,height=0.24\textheight]

\begin{figure*}[t]
	\centering	
	\subfigure[A  schematic of robot localization and navigation with twelve powered sensors]{
		\includegraphics[trim={3.5cm 4.9cm 11.7cm 1.5cm},clip,scale=0.3]{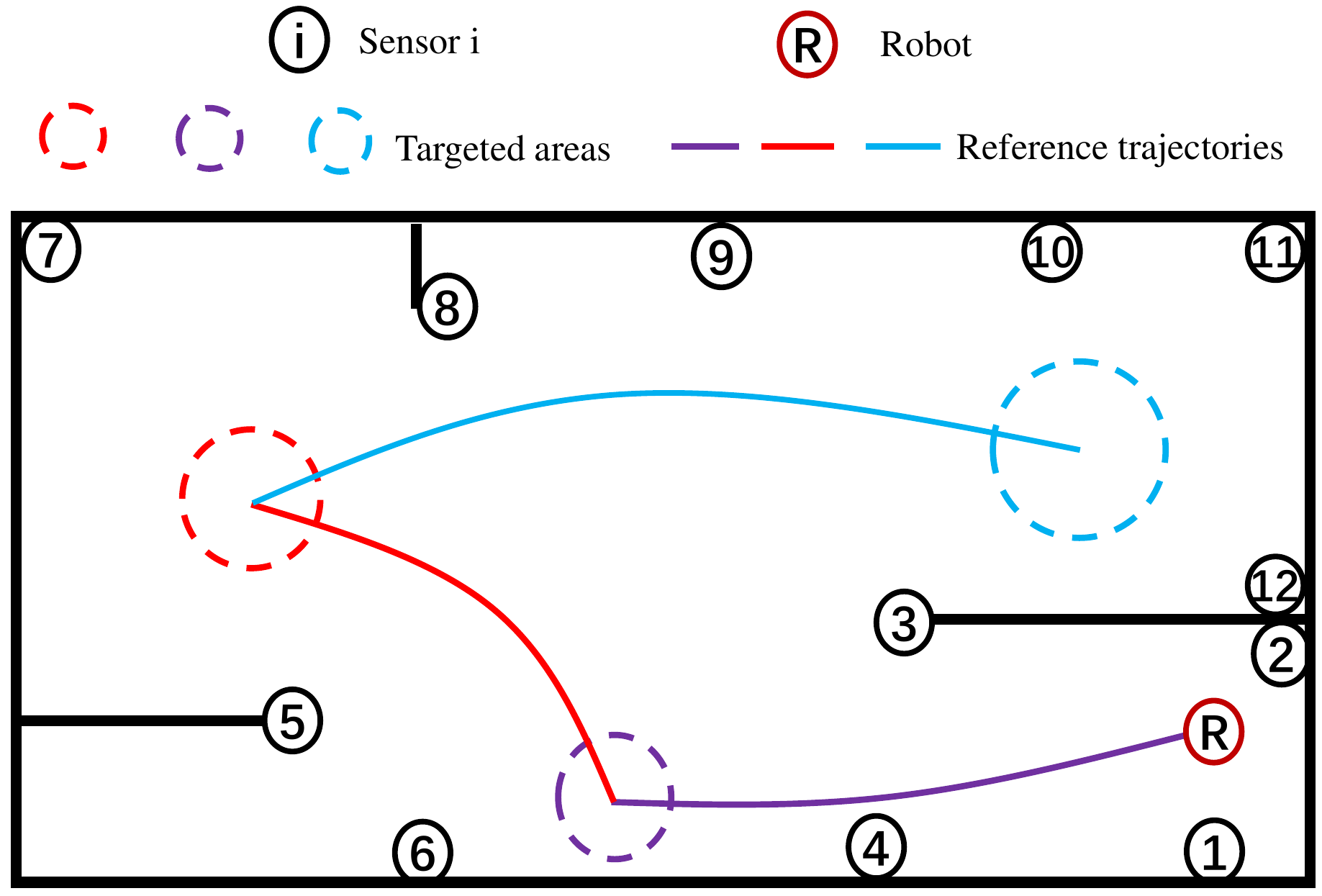}
	}		\hskip 5pt
	\subfigure[Active graph for task 1]{
		\raisebox{0.8\height}{\includegraphics[trim={0.4cm 11.9cm 19.2cm 2.9cm},clip,scale=0.25]{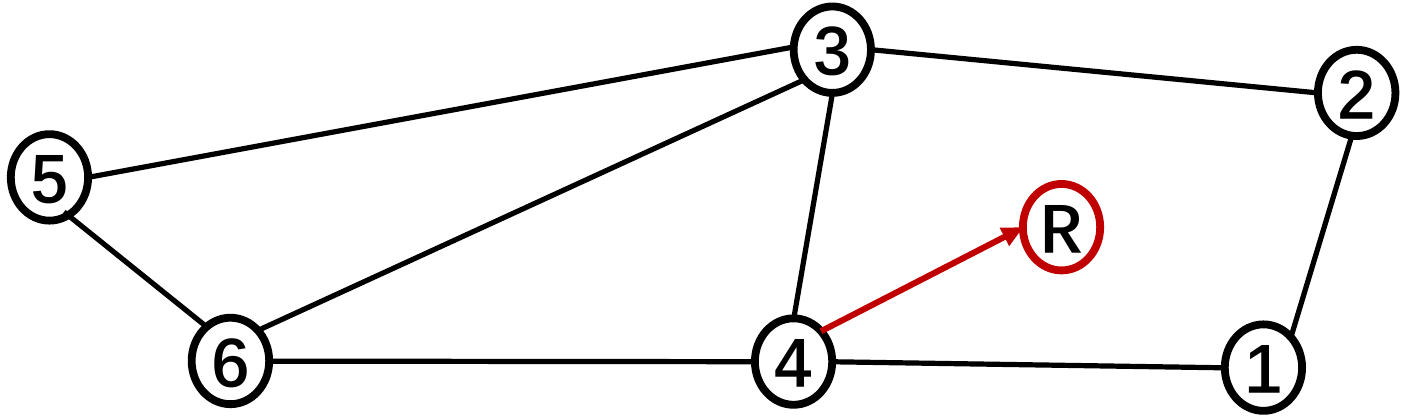}}
	}
\hskip 5pt
%\\\vskip10pt
	\subfigure[Active graph for task 2]{
		\includegraphics[trim={0.8cm 0.7cm 20.1cm 8.7cm},clip,scale=0.25]{robot_network.pdf}
	}
%	\hskip 60pt
\hskip 5pt
	\subfigure[Active graph for task 3]{
		\raisebox{0.25\height}{\includegraphics[trim={15.7cm 6.4cm 3.3cm 5cm},clip,scale=0.25]{robot_network.pdf}}
	}
	\caption{In a building with twelve powered sensors, a mobile robot aims at successively accomplishing three tasks, i.e., moving into  three targeted areas in order and staying inside at least for the desired time
		by following three reference trajectories, respectively, as shown in (a). The robot is able to activate a subset of the twelve sensors for the localization  in each task. The active sensors in each task form a DSN to sense and estimate the robot state collaboratively in order to  provide the localization information for the robot, as shown in (b)--(d) where the robot only receives messages from a single active sensor. 		 }
	\label{fig:robot}
\end{figure*}

 Since satellite navigation systems (e.g., GPS) may be unavailable or  significantly prone to error for indoor robot localization and navigation, building system designers can deploy 
 a large number of heterogeneous powered sensors (e.g., camera, radar, lidar) in different areas of the buildings.  The sensors can form DSNs to sense and estimate the robot state in a structure-robust and communication-efficient way.
 Under  resource and physical constraints, if an mobile robot is able to control the state of sensors (being active/inactive for sensing and estimation),   it can adaptively activate subsets of sensors for accomplishing a sequence of tasks, e.g., safety patrol,  goods delivery, and consumer guidance. 
 
 An example is provided in Fig. \ref{fig:robot} to illustrate the problem of the paper. In this example, we aim to navigate a mobile robot by using twelve powered sensors such that the robot can finish three tasks (i.e., enter three targeted areas and stay inside no less than the desired time)  successively by following the preset reference trajectories. 
  Since each sensor only has limited observability, 
 the twelve sensors form three DSNs to monitor the robot state collaboratively. 
 There are three periods corresponding to three tasks. In each period,  only one DSN is active in sensing,   communication  and estimation, and the rest of sensors are inactive. In the first (resp. second/third) period, the robot communicates with sensor 4 (resp. 5/9) from the DSN in (b) (resp. (c)/(d)) in order to obtain its state estimates. 
 The problem is  how to effectively use and switch the DSNs in order to navigate the robot for the accomplishment of the three tasks against system uncertainties.

\subsection{Model of the controlled robot }
Consider a mobile robot with the following   system dynamics
\begin{align}\label{sys}
\begin{split}
&x(t+1)=Ax(t)+Bu(t)+w(t), \quad t\in\mathbb{N}^+
\end{split}
\end{align}
where $x(t)\in\mathbb{R}^n$ is the unknown robot state, $u(t)\in\mathbb{R}^p$ the control input,
$w(t)\in\mathbb{R}^n$  the   bounded process uncertainty, i.e., there is a scalar $q_w\geq 0$, such that 
$\norm{w(t)}\leq q_w$, which can be used to model linearization error, unknown disturbance and uncertain dynamics. 
$A\in\mathbb{R}^{n\times n}$ and $B\in\mathbb{R}^{n\times p}$ are known system matrices. For more information of linear modeling of mobile robots as in \eqref{sys}, readers can refer to \cite{guerra2004linear}.

%$A\in \mathbb{R}^{n\times n}$ is the state transmission matrix, $B\in\mathbb{R}^{n\times p}$ is the control input matrix, and $C_i\in \mathbb{R}^{M_i\times n}$ is the measurement matrix of sensor $i$. 
%%$(A,C_i)$ is not required to be observable. 
%Assume   $w(t)$ is bounded, i.e., there is a scalar $q_w\geq 0$, such that 
%$\norm{w(t)}\leq q_w$. 
\begin{assumption}\label{ass_stabilization}
	The robot system is stabilizable, i.e., there is a matrix $K\in\mathbb{R}^{p\times n}$ such that $A+BK$ is Schur stable, i.e., all eigenvalues of $A+BK$ have norm strictly less than one.
\end{assumption}

We aim to design $u(t)$  such that the robot is able to accomplish a sequence of tasks defined  as follows.
\begin{definition}\label{def_task}
	The robot has accomplished  task $\eta\in \mathbb{N}^+$ at time $t\in \mathbb{N}^+$, if there is a time $t_{\eta}$ with $t\ge t_{\eta}\geq \mathbb{T}_{\eta}$, such that  
	$x(\bar t)\in \mathbb{O}(c_{\eta},R_{\eta})$, for  $ \forall \bar t\in[t_{\eta}-\mathbb{T}_{\eta}, t_{\eta}]$
	where 	
	$\mathbb{O}(c_{\eta},R_{\eta})=\{x\in\mathbb{R}^n|\norm{D_{\eta}x-c_{\eta}}\leq R_{\eta}\}$ is the $\eta$-th targeted set,
	  $\mathbb{T}_{\eta}\in \mathbb{N}^+$   the   dwell time,
	$c_{\eta}\in\mathbb{R}^{d_{\eta}}$ and $R_{\eta}\in\mathbb{R}^+$  the center and radius of the  targeted set,   $D_{\eta}\in\mathbb{R}^{d_{\eta}\times n}$    the constraint matrix, and $d_{\eta}$  a positive integer less than $n$.
	%	its  state enters $x(t)$ enters the ball $\mathbb{O}(c_{\eta},R_{\eta})$ and remain inside for $\mathbb{T}_{\eta}$ successive time instants at least,  $\eta\in\mathbb{N}$.
\end{definition}

%In the following, the targeted set of task $\eta$ refers to  set $\mathbb{O}(c_{\eta},R_{\eta})$. 

\begin{remark}
   Because of process uncertainties, it makes sense to define a task with a targeted set but not a point. If   dwell time $\mathbb{T}_{\eta}=1$, the problem   reduces to the controllability problem under uncertainties.   For different task $\eta$, parameters $\{\mathbb{T}_{\eta},D_{\eta}, c_{\eta},R_{\eta}\}$  may be totally differently.    The defined task can reduce to different tasks in different problems. For example, if $D_{\eta}=I$ and $c_{\eta}=0$, the task reduces to a stabilization task. For the second-order   system with  $A=\left(\begin{smallmatrix}
	1&T\\
	0&1
	\end{smallmatrix}\right)$ and $T>0$, if $D_{\eta}=(1,0)$, the task becomes how to control the robot's position into a certain area.
	A targeted set in Definition~\ref{def_task}   essentially refers to a multi-dimensional ball. The proposed methods in this paper can still be used for other types of targeted sets by partitioning  and approximating them via  a number of such balls.
\end{remark}

%The robot's objective   is to  reach $d$ targeted areas modeled by multi-dimensional balls successively, i.e., for $i=1,2,\dots,d$,  $$\mathbb{O}(c_i,R_i)=\{x\in\mathbb{R}^n|\norm{x-c_i}\leq R_i\},$$ 
%	where $c_i$ and $R_i$ are the center and radius of $i$-th targeted area.
%Since the goal is to accomplish 
Suppose there are $d\in\mathbb{N}^+$ task(s)  with task order $\eta=1,2,\dots, d$. 
%The $\eta$-th targeted set refers to set $\mathbb{O}(c_{\eta},R_{\eta})$.
 For convenience, task $\eta=0$ stands for the initialization.
%  with initial state $x_0\in \mathbb{O}(c_{0},R_{0})$.
If task $\eta=0,\dots,d-1$ is accomplished,   we need a reference trajectory to 
navigate the robot from targeted set $\mathbb{O}(c_{\eta},R_{\eta})$ to  targeted set $\mathbb{O}(c_{\eta+1},R_{\eta+1})$. For  $\eta=0,1,\dots,d-1$, suppose there is a reference trajectory, consisting of  reference states $\{r_{\eta,\eta+1}(l)\}_{l=1}^{T_{\eta,\eta+1}}$  and reference  inputs
$\{u_{\eta,\eta+1}(l)\}_{l=1}^{T_{\eta,\eta+1}}$, subject to 
\begin{align}\label{eq_reference}
\begin{split}
D_{\eta+1}r_{\eta,\eta+1}(1)&=c_{\eta}, \quad	D_{\eta+1}r_{\eta,\eta+1}(T_{\eta,\eta+1})=c_{\eta+1},\\
r_{\eta,\eta+1}(l+1)&=Ar_{\eta,\eta+1}(l)+Bu_{\eta,\eta+1}(l),
\end{split}
\end{align}
where $T_{\eta,\eta+1}\in\mathbb{N}^+$	is the length of the $(\eta+1)$-th reference trajectory, and $r_{0,1}(1)$ is the initial reference state to be determined in Assumption~\ref{ass_ini}.

\begin{remark}
	 The  reference trajectories can be generated with   interpolation methods (e.g., cubic splin interpolation as used in the simulation) or trajectory optimization, e.g., signal temporal logic in formal methods \cite{belta2019formal} and  direct collocation methods \cite{kelly2017introduction}.   Due to system uncertainties, it is not suggestible to replace   control input $u(t)$  by   the  reference inputs (as shown in the simulation). 
%	 In Section~\ref{subsec:control}, we will design a class of feedback control inputs.
\end{remark}
%\begin{remark}
%	The requirement in  \eqref{eq_reference} is to ensure the robot is able to accomplish a sequence of tasks one after another. 
%The proposed method can be simply extended to the case that the targeted sets of two adjacent tasks are not the same one, i.e.,  $, \quad	D_{\eta}r_{\eta-1,\eta}(T_{\eta-1,\eta})\neq D_{\eta+1}r_{\eta,\eta+1}(1)$. 
%\end{remark}

%Without losing generality, assume the robot starts with task~2, which means task 1 is just done at the initial time.  

\subsection{Model of sensor networks }
There are $N\geq 2$ (smart and powered) sensors which can be activated  to measure and estimate the robot's state.
Due to physical limitation or resource constraints, only a subset of sensors are active  at each time for sensing and estimation,  and the rest of sensors remain inactive.
The measurement equation of sensor $i$ at time $t$ is:
\begin{align}
y_i(t)=\delta_i(t)\left(C_ix(t)+v_{i}(t)\right), i=1,\dots,N,
\end{align}
where $y_i(t)\in\mathbb{R}^{M_i}$ is the measurement vector, $v_i(t)\in\mathbb{R}^{M_i}$  the  measurement uncertainty, $C_i\in\mathbb{R}^{M_i\times n}$  the measurement matrix, and $\delta_i(t)\in\{0,1\}$ is an indicator of sensor's state, where $\delta_i(t)=1$ if sensor $i$ is active at time $t\in\mathbb{N}^+$, otherwise $\delta_i(t)=0$. 
Assume the robot can determine  the values of $\{\delta_i(t)\}$ via communicating with sensors.  
Suppose   $v_{i}(t)$ is bounded, i.e., there is a scalar $q_v\geq 0$, such that 
$\norm{v_{i}(t)}\leq q_v$. Such measurement uncertainty can be used to model sensor bias, unknown  disturbance, and bounded noise  \cite{d2013bounded}. Assume that each sensor is able to store a set of parameters and the reference trajectory data $\{r_{\eta,\eta+1}(l)\}_{l=1}^{T_{\eta,\eta+1}}$, 
$\{u_{\eta,\eta+1}(l)\}_{l=1}^{T_{\eta,\eta+1}}$ with $\eta=0,1,\dots, d-1$ in advance.

%In  the traditional centralized architecture, a data center   collects the measurements of all sensors (i.e., $\{y_{i}(t),i=1,\dots,N\}$), based on which $u(t)$ is designed to control $x(t)$ for certain objectives. In this situation, the classical separation principle holds, which means the state estimation and system control problems can be separately studied. 
%However, with the increasing of sensor number and measurement size (e.g., camera images), centralized algorithms would confront inefficient scalability, and the 
%burden of the 
%communication channels between the data center and sensors would increase tremendously. However,  distributed schemes, which can reduce the communication between the data center and the sensors,  unfortunately, could make the  separation principle  fail. Thus, in a distributed scheme, co-design of a state estimator and a controller for the system \eqref{sys} is in need.
%% since  control center and the sensors have the asymmetric control information.
%% can be considered by employing the sensor network and
%%\subsection{Distributed communications of sensors}

%We consider a distributed communication scheme for the  sensor network, where each sensor simply shares information with its neighboring sensors. 

Since an individual sensor has limited observability (i.e.,  $(A,C_i)$ is unobservable or undetectable), the sensors are expected to estimate the robot state collaboratively. Suppose the $N$ sensors   form $Q_s$ DSNs.  We model the  communication of the  $k$-th network with $2\leq N(k)\leq N$ nodes   through a fixed undirected graph $\mathcal{G}(k)=(\mathcal{V}(k),\mathcal{E}(k),\mathcal{A}(k))$, where  $\mathcal{V}(k)\subset \{1,2,\dots,N\}$ denotes  the set of   nodes, $\mathcal{E}_k\subseteq \mathcal{V}(k)\times \mathcal{V}(k)$   the  set of  edges, and $\mathcal{ A}(k)$ the 0--1 adjacency matrix.  
If the $(i,j)$-th element of $\mathcal{ A}(k)$ is 0,   there is an edge $(i,j)\in \mathcal{E}(k)$, through which node $i$ can exchange information with node $j$. In the case,  node $j$  is called a  neighbor of node $i$, and vice versa. 
For $\mathcal{G}(k)$, let the neighbor set of node $i$ be $\mathcal{N}_{i}(k):=\{j\in\mathcal{V}(k)|(i,j)\in \mathcal{E}(k)\}$. Suppose that  each sensor has the knowledge of $\mathcal{N}_{i}(k)$ for $k=1,2,\dots,Q_s$.
 $\mathcal{D}(k)$ is the degree matrix, which is a diagonal matrix consisting of the numbers of neighbors. Denote $\mathcal{L}(k)=\mathcal{D}(k)-\mathcal A(k)$ the Laplacian matrix. 
The graph $\mathcal{G}(k)$ is  connected if for any pair of nodes $(i_{1},i_{l})$, there exists a  path from $i_{1}$ to $i_{l}$ consisting of edges $(i_{1},i_{2}),(i_{2},i_{3}),\ldots,(i_{l-1},i_{l})$. 
 It is   known that graph $\mathcal{G}(k)$ is  connected if and only if $\lambda_{2}(\mathcal{L}(k))>0$.
%In the \text{$k$-th}  mode, the set of active sensors is denoted by $\mathcal{V}(k)\subset\mathcal{V}$, $k=1,2,\dots,N_m$. 
%Denote $\mathcal{G}(t)=(\mathcal{V}(t),\mathcal{E}(t),\mathcal{A}(t))$ the active graph at time $t$, then we have $\mathcal{V}(t)=\{i\in\mathcal{V}|\delta_{i}(t)=1\}\in \{\mathcal{V}(k)\}_{k=1}^{Q_s}$. 
The following assumptions are needed in this paper.
\begin{assumption}\label{ass_detectability}
To accomplish $d$ tasks in order,  the robot is able to activate a sequence of connected graphs  $\{\mathcal{G}_{\eta}\}_{\eta=1}^{d}$ with $\mathcal{G}_{\eta}=\{\mathcal{V}_{\eta},\mathcal{E}_{\eta},\mathcal{A}_{\eta}\}\in \{\mathcal{G}(k)\}_{k=1}^{Q_s}$, such that  the system is collectively  detectable over $\mathcal{G}_{\eta}$ for each task $\eta$.
%is moving to each targeted set $\mathbb{O}(c_{\eta},R_{\eta})$, $\eta=1,\dots, d$,	there is at least a  connected sensor network $\mathcal{G}_{\eta}=\{\mathcal{V}_{\eta},\mathcal{E}_{\eta},\mathcal{A}_{\eta}\}\in \{\mathcal{G}(k)\}_{k=1}^{Q_s}$ such that if the network is active, then  the system will be collectively  detectable over $\mathcal{V}_{\eta}$.
\end{assumption}
%\begin{remark}
	The collective detectability in Assumption~\ref{ass_detectability} means that for each $\eta$, there is a matrix $G_{\eta}$ such that  $A-G_{\eta}\tilde C_{\eta}$ is Schur stable, where $\tilde C_{\eta}$ is  obtained by stacking all the measurement matrices of sensors in the set $\mathcal{V}_{\eta}.$  	The collective detectability is satisfied if there is one sensor $i$ such that $(A,C_i)$ is detectable, but not vice versa. 
	The detectability condition reduces to the mildest 
	one in distributed estimation when   network $\mathcal{G}_{\eta}$ includes  all sensors. Even if the system is not collectively detectable for some task $\eta$, one can partition this task into several tasks with length-reduced reference trajectories, then Assumption~\ref{ass_detectability} is satisfied as long as in each new task the collective detectability is ensured.
	
%	 For $\eta=1,\dots,d$, define 
%	$	\mathbb{G}(\eta)=\cup_{k=1}^{Q_s}\{\mathcal{G}_{k}|(A,\tilde C_{k}) \text{ is detectable for task } \eta\}.$
%	From Assumption \ref{ass_detectability},    $\mathbb{G}(\eta)$ is non-empty for each task $\eta$.
%	%	Let sensor $s_{\eta}\in\mathcal{V}_{\eta}$ be the sensor in Assumption \ref{ass_switching}  when network $\mathcal{G}_{\eta}$ is active for   task $\eta$. 
%\end{remark}

\begin{assumption}\label{ass_switching}
	The  robot has the full knowledge of the system and  sensor networks $\{\mathcal{G}_{\eta}\}_{\eta=1}^{d}$ satisfying Assumption~\ref{ass_detectability}.
%	, and  is able to persistently obtain data from  one sensor in the active sensor network  and 
%	determine which sensor network is active for each task. 
\end{assumption}
	Assumption~\ref{ass_switching} enables the robot to have sufficient information such that it can make decisions of when to switch to the next task.

\begin{assumption}\label{ass_ini}
	At the initial time,  it holds that 
$\norm{x(1)-\hat x_i(1)}\leq q_x,  
\norm{x(1)-r_{0,1}(1)}\leq q_r,$
	where $q_x$ and $q_r$ are non-negative scalars,   $\hat x_i(1)$ is the estimate of $x(1)$ by sensor $i$, $i=1,2\dots,N$, and $r_{0,1}(1)$ is  the initial reference state from \eqref{eq_reference}.
\end{assumption}

\subsection{Problem of interest}\label{subsec_prob}
We focus on solving the following subproblems. 

\begin{enumerate}
	\item How to design an integrated estimation and control architecture  such that the robot is able to successively accomplish $d$ tasks  by switching the active DSN? (Section~\ref{sec:architecture})
	
	\item How to evaluate the estimation error and trajectory tracking deviation in each task? (Theorem~\ref{thm_error_bound})
	
	\item What conditions can ensure that each task is successfully accomplished?
How to offline evaluate the running  time for each task? (Theorem~\ref{thm_control})
	
	%	\item Given task $i$, how to design $u(t)$ for the control center and  a distributed filter for each sensor, such that    state $x(t)$  well follows the reference trajectory
	%	$\{r_{i-1,i}(k)\}_{k=1}^{T_{i-1,i}}$, $i=1,2,\dots,k_*$?
	%	%	\item What conditions can ensure the state estimation error and  trajectory tracking deviation to be upper bounded?
	%	\item How to build the task switching scheme through which each task $i$ is surely accomplished (i.e., state $x(t)$ stays inside $\mathbb{O}(c_{i},R_{i})$ for $\mathbb{T}_{i}$ time instants at least) and then  task $i$ is switched to task $i+1$, $i=1,2,\dots,k_*$?
	%	\item How to evaluate the time cost in running each  task?
\end{enumerate}

\section{Task-Switching Architecture}\label{sec:architecture}
 In this section, we design an integrated task-switching  architecture, consisting of an event-triggered feedback   controller for the robot and a distributed estimator for each active sensor.

\subsection{Integrated estimation and control architecture}
Let $\mathcal{G}_{\eta}=(\mathcal{V}_{\eta},\mathcal{E}_{\eta},\mathcal{A}_{\eta})$ satisfying Assumption~\ref{ass_detectability} be the active sensor network for task $\eta\geq 1$. Then  the set of active sensors for task $\eta$ or $\eta+1$ is denoted as
\begin{align}\label{set_V}
\bar{\mathcal{V}}(\eta)=\mathcal{V}_{\eta-1}\bigcup\mathcal{V}_{\eta}.
\end{align}
Suppose   $\Delta_{\eta-1}(t)\geq 0$ is the condition   such that  the task is switched from $\eta-1$ to $\eta$ at time $t$. Then at this time instant, the robot broadcasts to each sensor in the set $\bar{\mathcal{V}}(\eta)$ with the following message  
\begin{align}\label{message_M}
\mathcal{M}_{\eta}=\{\eta,\hat x_{s_{\eta-1}}(t),s_{\eta}\},
\end{align}
where $\eta$ is the next  task index, and $\hat x_{s_{\eta-1}}(t)$ is the latest estimate of the robot received from sensor $s_{\eta-1}$ in  task $\eta-1$, and $s_{\eta}$ is the label of the   sensor transmitting estimates to the robot in task $\eta.$

The whole architecture is depicted in Fig. \ref{fig:diagram}. Communication in this architecture exists  in  three aspects:	(i).  The robot communicates   message $\bar{\mathcal{V}}(\eta)$ to a subset of sensors ($\bar{\mathcal{V}}(\eta)$) only at the moment when the task is switched.  
(ii).  The robot   receives its state estimates  from  one sensor. The sensor label $s_{\eta}$ could be different in different task $\eta$.
For example,  a  mobile robot can choose a near sensor to obtain its   state estimate.   (iii).  The active sensors can communicate with each other in a distributed manner over the activated sensor network.
Since the existing centralized architectures often require  that  the exact control signal $u(t)$    is shared with sensors all the time and that the robot communicates with all sensors for   obtaining their measurements, this architecture is more suitable when the system is under constrained resources or communication. 
	
%	 First, we study how to design an event-triggered task-switching control and then design a distributed filter to estimate the controlled robot state.
\subsection{Task-switching feedback  control}\label{subsec:control}
%	\begin{figure}[t]
%		\centering
%		\includegraphics[scale=0.34]{diagram_event.eps}
%		\caption{Target-switched control based on distributed estimation}
%		\label{fig:diagram}
%	\end{figure}

\begin{figure}[t]
	\centering
\includegraphics[scale=0.28]{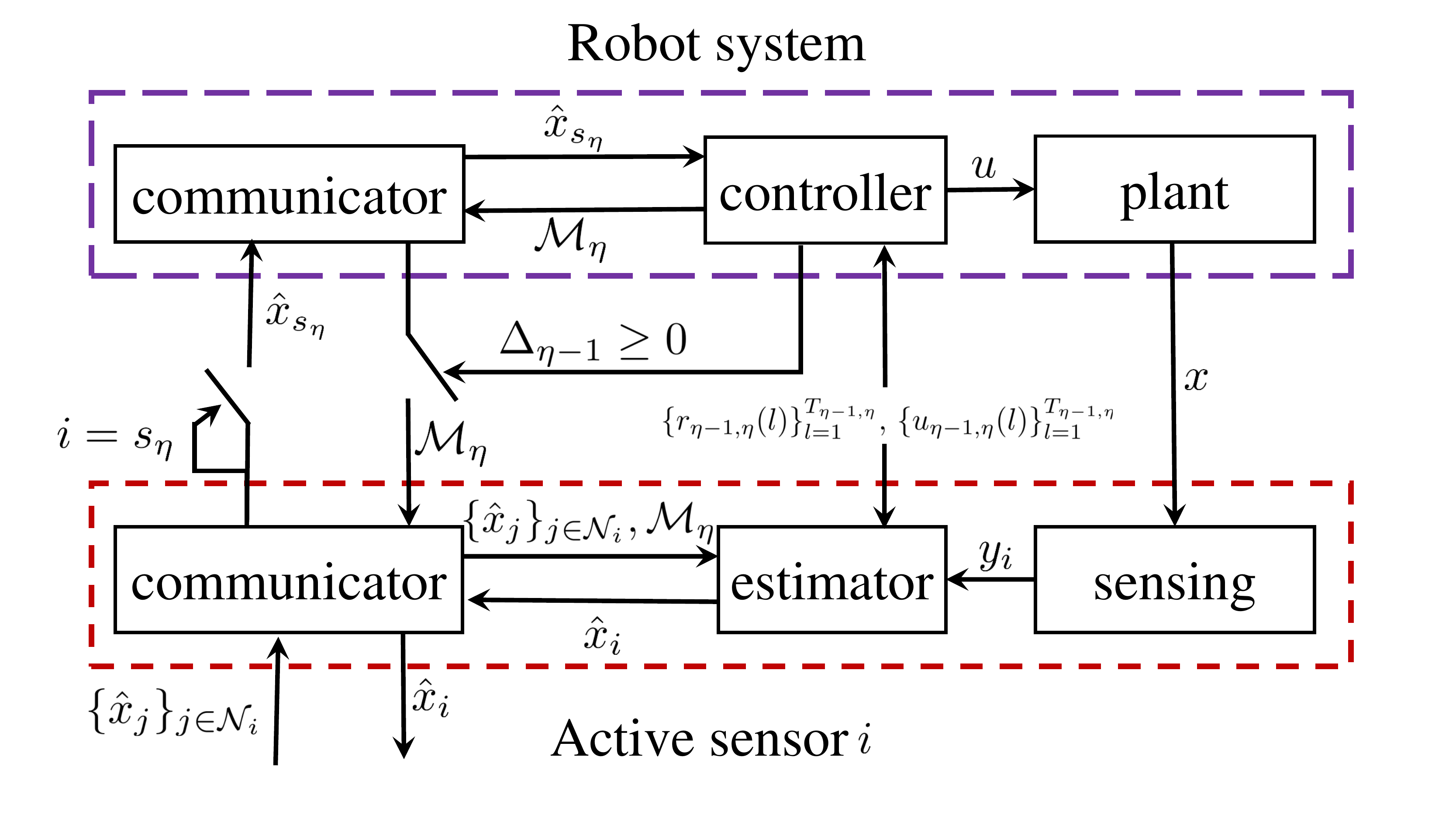}
	\caption{An integrated estimation and control  architecture with event-triggered task switching. When  $\Delta_{\eta-1}\geq 0$, the robot broadcasts $\mathcal{M}_{\eta}$ to  sensor $i\in \bar{\mathcal{V}}(\eta)$. If  sensor $i$ knows  from $\mathcal{M}_{\eta}$ that it should be active in this task, to update its estimate $\hat x_i$, it will use  measurement $y_i$, $\mathcal{M}_{\eta}$, the reference trajectory data $\{r_{\eta-1,\eta}(l)\}_{l=1}^{T_{\eta-1,\eta}}$  and
		$\{u_{\eta-1,\eta}(l)\}_{l=1}^{T_{\eta-1,\eta}}$, and the estimates from neighbor active sensors. If  sensor $i$ is inactive in this task, it will shut down the estimator and sensing components until it is activated again. Sensor $s_{\eta}$   communicates its estimate $\hat x_{s_{\eta}}$ with the robot for the controller design.		
		}
	
	\label{fig:diagram}
\end{figure}

For  task $\eta\geq 1$, according to the architecture in Fig.~\ref{fig:diagram}, 
we design the following feedback control input $u(t)$   with a fixed control gain $K\in\mathbb{R}^{p\times n}$: 
\begin{align}\label{control}
\begin{split}
u(t)&=K\left(\hat x_{s_{\eta}}(t)-r(t)\right)+u_r(t),\\
r(t)&=r_{\eta-1,\eta}(t-t_{\eta-1}+1),\\
u_r(t)&=u_{\eta-1,\eta}(t-t_{\eta-1}+1),
\end{split}
\end{align}
where   
%$r_{\bar\eta(t),\eta(t)}(t)$ is a reference trajectory connecting the centers of the  task ball $O(c_{\bar\eta(t)},R_{\bar\eta(t)})$ and the task ball $O(c_{\eta(t)},R_{\eta(t)})$.
$t_{\eta-1}$ is the ending time for task $\eta-1$ (also the initial time for 
task $\eta$), i.e., the 
 time when $\Delta_{\eta-1}(t)\geq 0$.

%\begin{remark}
%	Since the reference`1 trajectories are   in order, given a sequence of $t_{\eta}$, the variable $\eta$ is able to be updated correspondingly. Thus , $r(t)$ and $u_r(t)$ in \eqref{control} are determined by simply sharing $t_{\eta}$ between the robot and the sensor network.
%\end{remark}
%% and $u_{\bar\eta(t),\eta(t)}(t)\in\mathbb{R}^{p}$  the control input corresponding to the trajectory $r_{\bar\eta(t),\eta(t)}(t)$.

\begin{remark}
	 Since $u(t)$ in \eqref{control} does not involve complex computations, it is  suitable  for the robot with limited computational resources.
	Compared with  $u_r(t)$ as the control input signal,   feedback   control signal $u(t)$    in \eqref{control} has advantages  against   system uncertainties.
\end{remark}

\begin{algorithm}[t]
	\small
	\caption{Event-triggered task-switching controller}
	\label{alg:B}
	\begin{algorithmic}
		\STATE{\textbf{Initial setting:} Control gain $K$,   targeted sets $\{\mathbb{O}(c_{\eta},R_{\eta})\}_{\eta=1}^{d}$, reference states $\{r_{\eta-1,\eta}(l)\}_{l=1}^{T_{\eta-1,\eta}}$ and inputs $\{u_{\eta-1,\eta}(l)\}_{l=1}^{T_{\eta-1,\eta}}$, $\eta=1,\dots,d$, desired dwell time in  targeted sets $\{\mathbb{T}_{\eta}\}$, cumulative  remaining time  set $\{\mathcal{T}_{\eta}\}$ initialized by empty sets, task accomplishment indicators $\mathcal{J}_{\eta}=0$ for each $\eta$,    initial task label $\eta=1$;  }
		\FOR{$t=1,2,\dots$}
		\STATE{ \textbf{// Communication with sensor $s_{\eta}$:}\\
			Robot obtains an estimate $\hat x_{s_{\eta}}(t)$ from sensor $s_{\eta}$}; \\\vskip 5pt
		\STATE{	{\textbf{// Event-triggered condition for task switching:}} 
			\IF {$f(t)\leq R_{\eta}$} 
			\STATE{ $\mathcal{T}_{\eta}=\mathcal{T}_{\eta}\bigcup \{t\}$}
			\ENDIF	
			\IF {$|\mathcal{T}_{\eta}|= \mathbb{T}_{\eta}$ }  
			\STATE{ $\mathcal{J}_{\eta}=1\qquad$ // task $\eta$ is accomplished}  
			\STATE{$t_{\eta}=t$ \qquad// task switching time instant}
			%			\ELSIF{$t=T_{\eta,\eta+1}+t_{\eta}$} 
			%			\STATE{ $\mathcal{J}_{\eta}=0,$ $\eta=\eta+1$, 
			%				$\mathbb{T}_{\eta}=t+1$;} 
			\ENDIF					
		}		\vskip 5pt
		\STATE{ \textbf{//  Communication with some sensors when   task $\eta$ is accomplished}}
		\IF {$\mathcal{J}_{\eta}=1$}  
		\STATE{  
			Broadcast   message $\mathcal{M}_{\eta+1}$ in \eqref{message_M} to sensors in the set $\bar{\mathcal{V}}(\eta+1)$ \eqref{set_V};\\
			$\eta=\eta+1$;
		}  
		\ENDIF					\vskip 5pt			 	
		%		\STATE{$\mathcal{G}(t+1)=\mathcal{G}(t);$ }
		\STATE{	{\textbf{// Feedback control input:}} \\ 
			Design $u(t)$ as in \eqref{control}};
		\ENDFOR
	\end{algorithmic}
\end{algorithm}

%The design objective of the control input $u(t)$ is to drive the unknown system state $x(t)$ into the task balls $\{\mathbb{O}(c_{i},R_i)\}_{i=1}^{k_*}$ one after another by following the reference trajectory  $\{r_{\eta-1,\eta}(t)\}$, where $\eta\in\{2,\dots, k_*\}$. 
%When 
%\begin{align}\label{event0}
%\norm{x(t)-c_{\eta}}\leq R_{\eta}, 
%\end{align}
%is satisfied, then the task ball and the reference trajectory will be switched to the next one. 

Condition $\Delta_{\eta-1}(t)\geq 0$ is supposed to ensure that task $\eta-1$ is accomplished for sure.
 The following lemma provides a sufficient condition such that the robot is in a targeted set at a single time.
\begin{lemma}\label{lem_event}
	Suppose there are two sequences $\{h_e(t)\}$ and $\{h_c(t)\}$,  such that 
	$$\norm{D_{\eta}\left(x(t)-\hat x_{s_{\eta}}(t)\right)}\leq h_e(t),\quad \norm{D_{\eta}\left(x(t)-r(t)\right)}\leq h_c(t).$$
	Then robot state $x(t)$ is in the targeted set $\mathbb{O}(c_{\eta},R_{\eta})$, if 
	\begin{align}\label{evet_condi}
	f(t):=\min\big\{h_e(t)+g_e(t), h_c(t)+g_c(t)\big\}\leq R_{\eta},
	\end{align}
	where $	g_e(t)=\norm{D_{\eta}\hat x_{s_{\eta}}(t)-c_{\eta}}$ and $g_c(t)=\norm{ D_{\eta}r(t)-c_{\eta}}.$
\end{lemma}
\begin{proof}
It holds that $D_{\eta}x(t)-c_{\eta}=D_{\eta}x(t)-D_{\eta}\hat x(t)+D_{\eta}\hat x(t)-c_{\eta},
D_{\eta}x(t)-c_{\eta}=D_{\eta}x(t)-D_{\eta}r(t)+D_{\eta}r(t)-c_{\eta}.$
%\begin{align*}
%D_{\eta}x(t)-c_{\eta}&=D_{\eta}x(t)-D_{\eta}\hat x(t)+D_{\eta}\hat x(t)-c_{\eta}\\
%D_{\eta}x(t)-c_{\eta}&=D_{\eta}x(t)-D_{\eta}r(t)+D_{\eta}r(t)-c_{\eta}.
%\end{align*}
By taking 2-norm of both sides of the above equalities and applying the norm  triangle inequality, we have 
$\norm{D_{\eta}x(t)-c_{\eta}}\leq f(t)$. Then if \eqref{evet_condi} holds, $\norm{D_{\eta}x(t)-c_{\eta}}\leq   R_{\eta}$, which means 
 $x(t)\in \mathbb{O}(c_{\eta},R_{\eta}).$  
	\end{proof}

Since $\{g_e(t)\}$ and $\{g_c(t)\}$ can be directly computed by the robot,   condition \eqref{evet_condi} can be used in the design of $\Delta_{\eta-1}(t)\geq 0$ provided with  two sequences $\{h_e(t)\}$ and $\{h_c(t)\}$ satisfying the requirement in Lemma~\ref{lem_event}. In Section~\ref{sec:theory}, we  provide a design of $\{h_e(t)\}$ and $\{h_c(t)\}$.  
Based on the proposed architecture in Fig.~\ref{fig:diagram}, control input \eqref{control}, and Lemma~\ref{lem_event},  
we propose an event-triggered task-switching controller  (Algorithm~\ref{alg:B}) such that the   robot is able to accomplish the given tasks in sequence.

\subsection{Task-switching distributed estimator}
%%%%%%%%%%%%%%%%%%%%%%%%%%%%%%%%%%%%%%%%%%%%%%%%%%%%%%
%In this subsection, 
In order to provide the robot with  state estimates over the active DSN in each task $\eta$, we  propose a two-time-scale    distributed estimator (Algorithm~\ref{alg:A})
for each  sensor $i\in \bar{\mathcal{V}}(\eta)$. 

\begin{algorithm}[t]
	\small
	\caption{Task-switching distributed estimator}
	\label{alg:A}
	\begin{algorithmic}
		\STATE{\textbf{Initial setting:} Control gain $K$,     reference states $\{r_{\eta-1,\eta}(l)\}_{l=1}^{T_{\eta-1,\eta}}$ and inputs $\{u_{\eta-1,\eta}(l)\}_{l=1}^{T_{\eta-1,\eta}}$, $\eta=1,\dots,d$,  parameter set $\{G_{i,\eta},L_{\eta},\alpha_{\eta}\}_{\eta=1}^{d}$, and initial estimate $\hat x_i(1)$;  }
		\FOR{$t=1,2,\dots$}	       
		\IF{Sensor $i$ receives $\mathcal{M}_{\eta}$ and $i\notin\mathcal{V}_{\eta} $}
		\STATE{Sensor $i$ becomes inactive in sensing and estimation;}
		\ELSIF{Sensor $i$ receives $\mathcal{M}_{\eta}$ and $i\in\mathcal{V}_{\eta} $}
		\STATE{Based on $\mathcal{M}_{\eta}$ in \eqref{message_M}, update $\eta,s_{\eta}$, and let $\hat x_{i}(t)=\hat x_{s_{\eta-1}}(t)$.}		
		\ELSE
		\STATE{	\textbf{// Predicted Controller:}  \\
			%			Each sensor $i$ uses    $t_{\eta}$  from the  center, 
			$\hat u_i(t)=K\left(\hat x_{i}(t)-r(t)\right)+u_r(t)$
			%			r(t)&=r_{\eta(t),\eta(t)+1}(t-t_{\eta(t)}+1),\\
			%			u_r(t)&=u_{\eta(t),\eta(t)+1}(t-t_{\eta(t)}+1);
			where $r(t)$ and $u_r(t)$ are given in \eqref{control}.		}\\\vskip 5pt
		\STATE{	\textbf{// Estimator Update:}}\\
		%		$\bar x_{i}(t)=(A+K)\hat x_{i}(t-1)$\\
		$		\tilde x_{i}(t)=A\hat x_{i}(t)+B\hat u_i(t) + G_{i,\eta}(y_{i}(t)-C_i\hat x_{i}(t)),$
		where $G_{i,\eta}$ is the corresponding gain of sensor $i$ when $\mathcal{G}_{\eta}$ is active.
		\vskip 5pt
		\STATE{		\textbf{// Neighbor sensors communicate for $L_{\eta}$ times:}} \\
		\FOR{  $l=1,\dots,L_{\eta}$}
		\STATE{	 Sensor $i$ receives $\bar x_{j,l-1}(t)$ from neighbor sensor $j$, and runs
			$\bar x_{i,l}(t)=\bar{x}_{i,l-1}(t) -\alpha_{\eta}\sum_{j\in\mathcal{N}_{i,\eta}}(\bar{x}_{i,l-1}(t)-\bar{x}_{j,l-1}(t)),$
			where $\bar x_{i,0}(t)=\tilde x_{i}(t)$, $\hat x_{i}(t+1)=\bar x_{i,L_{\eta}}(t)$, and $\mathcal{N}_{i,\eta}$ is the neighbor set of sensor $i$ in graph $\mathcal{G}_{\eta};$
		}
		\ENDFOR
		\STATE{ \textbf{//  Communication with the robot}}
		\IF{ $i=s_{\eta}$}
		\STATE{	  Sensor $i$ transmits its estimate $\hat x_{s_{\eta}}(t+1)$ to the robot.}
		\ENDIF			
		\ENDIF
		\ENDFOR
	\end{algorithmic}
\end{algorithm}

%\begin{algorithmic}
%	\STATE $i\gets 10$
%	\IF {$i\geq 5$} 
%	\STATE $i\gets i-1$
%	\ELSE
%	\IF {$i\leq 3$}
%	\STATE $i\gets i+2$
%	\ENDIF
%	\ENDIF 
%\end{algorithmic}

	\begin{remark}
Based on the task information (i.e., $\eta$  in $\mathcal{M}_{\eta}$) from the robot, each active sensor is able to choose the corresponding reference input  $u_r(t)$ and  parameters $\{G_{i,\eta},L_{\eta},\alpha_{\eta}\}.$   Integer $L_{\eta}$ represents the communication times of each sensor with its  neighboring active sensors between two measurement updates. 
%The  parameter $\alpha_{\eta} \in\mathbb{R}^+$ is a consensus parameter, which influences the consensus speed of state estimates of sensors. 
%It can be proved that the estimates of all active sensors will reach consensus when $L_k$ goes to infinity provided that   the network is connected and     $\alpha_{\eta}\in(0,2/\lambda_{\max}(\mathcal{L}_{\eta}))$.
In this paper, $L_{\eta}$ is not necessarily very	large. For each task, an explicit requirement of $L_{\eta}$  is given in Theorem~\ref{thm_error_bound}. Such a two-time-scale distributed estimation can be supported with the advancement of communication technologies like 5G. In the literature, a number of two-time-scale distributed algorithms have been studied (e.g., \cite{marelli2021distributed}).
	\end{remark} 
	\begin{remark}
		Since sensor $s_{\eta}$ shares its estimate with the robot, its estimation error    is due to system uncertainties. 
		 For the rest of sensors,   their estimation errors are also affected by   inexact control inputs.
		The   errors induced from the inexact control inputs will be reduced if a large $L_{\eta}$ is designed.   
	\end{remark}

\begin{lemma}\label{lem_parameter}
	Under Assumptions \ref{ass_stabilization}--\ref{ass_switching},  it is feasible to choose 
	 parameters $\{G_{i,\eta},L_{\eta},\alpha_{\eta}\}_{\eta=1}^{d}$ for Algorithms \ref{alg:B}--\ref{alg:A} such that:
	\begin{itemize}
		\item Matrix $A-G_{\eta}\tilde C_{\eta}$ is Schur stable, where $\tilde C_{\eta}$ and $G_{\eta}$ are  obtained by stacking  measurement matrices $\{C_{i}\}_{i\in\mathcal{V}_{\eta}}$ and weighted gain matrices $\{G_{i,\eta}/N_\eta\}_{i\in\mathcal{V}_{\eta}}$  in row and column, respectively.
		\item Control gain matrix $K$ satisfies Assumption \ref{ass_stabilization}.
		\item Communication parameter $\alpha_{\eta}\in(0,2/\lambda_{\max}(\mathcal{L}_{\eta}))$.
	\end{itemize}
\end{lemma}

%\begin{remark}
%	The architecture consisting of Algorithms \ref{alg:B} and \ref{alg:A} is easy to run in the sense that  no  online optimization problems need to be solved. Moreover,  the communication between the sensor network and the robot meets the restriction in Assumption \ref{ass_all}.
%\end{remark}

%\section{Theoretical Results}

%The proposed event-triggered task-switching architecture, comprising Algorithms~\ref{alg:B} and \ref{alg:A}, provides a promising and effective solution for   state estimation and navigation of robots (e.g., mobile robots) monitored over DSNs.

%\begin{proposition}\label{prop_feasible}
%Under Assumptions \ref{ass_noise}, \ref{ass_observ} and \ref{ass_graph}, the condition \eqref{condition} is feasible, i.e., there exists an $L$ and a set of $\{G_i\}_{i=1}^N$ such that $ S(L,\{G_i\}_{i=1}^N)<1$.
%\end{proposition}
%\begin{proof}
%	See the proof in Appendix.
%\end{proof}

%In the following lemma, a design for the consensus parameter $\alpha$ is  provided.
%\begin{lemma}\label{lem_alpha}\cite{he2019secure}
%	Under Assumption \ref{ass_graph}, the following result holds
%	\begin{align}
%	\min_{\alpha}\norm{I_{Nn}-\alpha(\mathcal{L}\otimes I_n)-\bar I_{k}}_2
%	=&\frac{\lambda_{max}(\mathcal{L})-\lambda_2(\mathcal{L})}{\lambda_{max}(\mathcal{L})+\lambda_2(\mathcal{L})}<1,\nonumber
%	\end{align}
%	with the optimal  solution $\alpha^*=\frac{2}{\lambda_2(\mathcal{L})+\lambda_{max}(\mathcal{L})}.$
%\end{lemma}
%\input{performance}

\section{Performance Analysis}\label{sec:theory}
In this section, we aim to find upper bounds of the estimation error (i.e., $\norm{x_i(t)-\hat x_i(t)}$) and the trajectory tracking deviation (i.e., $\norm{x_i(t)-r(t)}$), and then provide a design of $h_e(t)$ and $h_c(t)$ for    event-triggered task-switching condition \eqref{evet_condi}.
 Moreover, with this design, we find the condition ensuring the accomplishment of each task. Furthermore,   an upper bound of running time for each task is provided.

%aim to find the conditions such that the state estimation error and trajectory tracking deviation  are both bounded, and an upper bound of the running time of  each task is obtained.

%Under the mildest system condition, i.e., collective detectability, Theorem \ref{thm_error_bound} studies the boundedness of the state estimation error with the proposed architecture.
%The upper bound $ h_e(t)$ in (i) of Theorem \ref{thm_error_bound} is useful to the design of the event-triggered threshold $f(t)$ in Algorithm~\ref{alg:B}. The conclusions in (ii) and (iii) of Theorem \ref{thm_error_bound}   contribute to the offline precision evaluation, and can
%improve  the accuracy and   reliability of practical  systems like robot navigation and localization.

%\subsection{Design of $h_e(t)$ and $h_c(t)$ for task switching}
\subsection{Upper bounds of estimation error and tracking deviation}\label{sub_bounds}
 For further analysis,  
we compute two sequences $\{\bar h_e(t)\}$ and $\{\bar h_c(t)\}$ by Algorithm~\ref{alg:bounds}.
To compute $\{\bar h_e(t)\}$ in Algorithm~\ref{alg:bounds}, denote 
\begin{align}\label{eq_he}
a_{\eta}(t)&=\frac{(1+4\gamma_{\eta}^2)N_{\eta}\lambda_{\max}(P_{\eta})}{\lambda_{\min}(P_{\eta})}  \varpi_{\eta}^{t-t_{\eta-1}}\nonumber\\
%		P_k&=\sum_{i=0}^{\infty}((A(k)-G(k)\tilde C(k))^i)^{\sf T}(A-G(k)\tilde C(k))^i\nonumber\\
b_{\eta}(t)&=\frac{q_{\eta}\sum_{l=0}^{t-t_{\eta-1}-1}\varpi_{\eta}^l}{\lambda_{\min}(P_{\eta})}, \nonumber\\
\varpi_{\eta}&=1-1/(3\lambda_{\max}(P_{\eta})) \nonumber\\
q_{\eta}&=(1+1/\rho_{\eta})\lambda_{\max}(P_{\eta})N_{\eta}\bar q_{\eta}\\
\bar q_{\eta}&=\big(q_w+q_v\sum_{i\in\mathcal{V}_{\eta}}\norm{G_{i,k}}/N_{\eta}\big)^2 +q_v^2\gamma_{\eta}^2\lambda_{c,\eta}^{2L_{\eta}}\max_{i\in\mathcal{V}_{\eta}}\norm{G_{i,k}}^2\nonumber\\
\lambda_{c,\eta}&=\max\{|1-\alpha_{\eta}\lambda_{2}(\mathcal{L}_{\eta})|,|1-\alpha_{\eta}\lambda_{\max}(\mathcal{L}_{\eta})|\}, \nonumber
%	h_e(t_{1})&\geq  \norm{e_{\max}(0)}.\nonumber
\end{align} 
where $\gamma_{\eta}$ and $\rho_{\eta}$ are any positive scalars, and 	
$P_{\eta}$   is  the unique solution  to the following Riccati equation  
\begin{align}\label{eq_riccati}
(A-G_{\eta}\tilde C_{\eta})^{\sf T}P_{\eta}(A-G_{\eta}\tilde C_{\eta})+I_n=P_{\eta}.
\end{align}	

%Given $\eta=1,2,\dots,d$, for $t\in (t_{\eta-1},t_{\eta}]$, construct 	$\{\bar h_e(t)\}$  and $\{\bar h_c(t)\}$ in the following way	
%Define	$\bar h_c(1)=q_r$. 
%For $t\in[t_{\eta-1}+1,t_{\eta}]$ with $\eta\geq 1$ and $t_{0}=1$, define 
%Define
%	\begin{align}\label{h_c}
%		\bar h_c(t)
%	=\sqrt{\hat a_{\eta}(t)\left(\bar h_c(t_{\eta-1})+\hat r (t_{\eta-1})\right)^2 +\hat b_{\eta}(t)}, 
%	\end{align}
% where 
To compute $\{\bar h_c(t)\}$ in Algorithm~\ref{alg:bounds}, denote 
	\begin{align}\label{eq_control_parameters}
%\begin{split}
	\hat r(t_{\eta-1})&=\norm{r_{\eta-1,\eta}(1)-r(t_{\eta-1})}\nonumber\\
\hat a_{\eta}(t)&=\frac{\lambda_{\max}(P_*)}{\lambda_{\min}(P_*)}\lambda^{t-t_{\eta-1}}\nonumber\\
\lambda&=1-1/(2\lambda_{\max}(P_*))\\
\hat b_{\eta}(t)&=\frac{\beta}{\lambda_{\min}(P_*)}\sum_{l=0}^{t-t_{\eta-1}-1}\lambda^{l}(\norm{BK}\bar h_e(t-1-l)+q_w)\nonumber\\
\beta&=\norm{P_*}+2\norm{P_*(A+BK)}^2,\nonumber
%\end{split}
	\end{align}
 and $P_*$ is  the unique solution  to  the following Riccati equation 
 \begin{align}\label{eq_riccati2}
 (A+BK)^{\sf T}P_{*}(A+BK)+I_n=P_{*}.
 \end{align}
 
 \begin{remark}\label{rem_random}
 	  Since the robot has the knowledge of $\{t_{\eta-1}\}_{\eta=2}^{d}$, under Assumption~\ref{ass_switching}, it is able to compute   sequences $\{\bar h_e(t)\}$ and $\{\bar h_c(t)\}$ with Algorithm~\ref{alg:bounds}. Note that $\{t_{\eta-1}\}_{\eta=2}^{d}$ are random according to the event-triggered task switching, thus  $\{\bar h_e(t)\}$ and $\{\bar h_c(t)\}$ are also random.
 \end{remark}
 
 \begin{algorithm}[t]
 	\small
 	\caption{Computation of $\{\bar h_e(t)\}$ and $\{\bar h_c(t)\}$}
 	\label{alg:bounds}
 	%\hspace*{\algorithmicindent} \textbf{Initial setting:}
 	%	\hspace*{\algorithmicindent} \textbf{Output: $\{\bar h_e(t)\}$ and $\{\bar h_c(t)\}$}
 	\begin{algorithmic}
 		%		\STATE{\textbf{Initial setting:} Control gain $K$,     reference states $\{r_{\eta-1,\eta}(l)\}_{l=1}^{T_{\eta-1,\eta}}$ and inputs $\{u_{\eta-1,\eta}(l)\}_{l=1}^{T_{\eta-1,\eta}}$, $\eta=1,\dots,d$,  parameter set $\{G_{i,\eta},L_{\eta},\alpha_{\eta}\}_{\eta=1}^{d}$, and initial estimate $\hat x_i(1)$;  } 
 		\STATE{\textbf{Initial setting:} 
 			$\{t_{\eta-1}\}_{\eta=1}^d$ with $t_0=1$ from Algorithm~\ref{alg:B}, and $a_{\eta}(t), b_{\eta}(t)$ from \eqref{eq_he}, and $\hat a_{\eta}(t),	\hat r(t_{\eta-1})	\hat b_{\eta}(t)$ from \eqref{eq_control_parameters}
 		}
 		\FOR{$t=1,2,\dots$}	  
 		\IF{$t=1$}
 		\STATE{$\bar h_e(1)=q_x,\quad \bar h_c(1)=q_r,\quad \eta=1$} 
 		\ELSIF{$t\in[2,t_1]$}
 		\STATE{$\bar h_e(t)=\sqrt{ a_{\eta}(t-1)q_x^2+b_{\eta}(t-1)}$ \\ $\bar h_c(t)
 			=\sqrt{\hat a_{\eta}(t)\left(\bar h_c(t_{\eta-1})+\hat r (t_{\eta-1})\right)^2 +\hat b_{\eta}(t)}, $}
 		\ELSIF{$t\in[t_{\eta-1}+1,t_{\eta}]$}
 		\STATE{$\bar h_e(t)=\sqrt{a_{\eta}(t) \bar h_e^2(t_{\eta-1})+b_{\eta}(t)}$\\ $\bar h_c(t)
 			=\sqrt{\hat a_{\eta}(t)\left(\bar h_c(t_{\eta-1})+\hat r (t_{\eta-1})\right)^2 +\hat b_{\eta}(t)}, $}
 		\ENDIF	  
 		\IF{$t=t_{\eta}$}
 		\STATE{$\eta=\eta+1$}
 		\ENDIF	  
 		\STATE{$t=t+1$}   
 		\ENDFOR 
 	\end{algorithmic}
 \end{algorithm}
 
	Next we provide conditions such that the two sequences $\{\bar h_e(t)\}$ and $\{\bar h_c(t)\}$   are   upper bounds of the state estimation error and the trajectory tracking deviation respectively. 
\begin{theorem}\label{thm_error_bound}
	Consider the  system and Algorithms~\ref{alg:B}--\ref{alg:bounds} satisfying Assumptions \ref{ass_stabilization}--\ref{ass_ini} with      parameters $\{G_{i,\eta}\}$, $K$, and $\{\alpha_{\eta}\}$   designed as in Lemma~\ref{lem_parameter}.  
	The estimation error of each active sensor and the trajectory tracking deviation are both upper bounded, i.e., for any $t\in [t_{\eta-1}+1,t_{\eta}]$, 
	\begin{align}\label{bound_error} 
	\norm{\hat x_i(t)-x(t)}\leq \bar h_e(t),  \quad 
	\norm{ x(t)-r(t)}\leq	\bar h_c(t), 
	\end{align}
if  for any task $1\leq j\leq \eta$    communication parameter 	$L_{j}$ is subject to 
	\begin{align}\label{L_bound}
	L_{j}\geq \max\{d_{j},e_{j}\}
	\end{align}
	where 
	\begin{align*}
	d_{j}&=\frac{\log \left(\gamma_{j}\hat q_{j}\sqrt{3(1+\rho_{j})(1+\tau_{j})}\right)}{\log \lambda_{c,{j}}^{-1}}\\
	e_{{j}}&=\frac{\log \left((\norm{A+BK}+\hat q_{j})\sqrt{3(1+\rho_{j})(1+1/\tau_{j})}\right)}{\log \lambda_{c,{j}}^{-1}}
	\end{align*}
	in which $\hat  q_{j}=\max_{i\in\mathcal{V}_{j}}\norm{G_{i,{j}}C_i},$ and  positive scalars	$\rho_{j}$, $\tau_{j}$, and  $\gamma_{j}$ are  subject to
	\begin{align}\label{eq_rho_gamma}
\begin{split}
	&(1+\rho_{j})(1+\tau_{j})\leq \frac{\lambda_{\max}(P_{j})-2/3}{\lambda_{\max}(P_{j})-1}\\
	&\gamma_{j}\geq \max\{\sqrt{2},f_{j}\sqrt{3(1+\rho_{j})(1+1/\tau_{j})\lambda_{\max}(P_{j})}\},
\end{split}
	\end{align}
	where $f_{j}=\sqrt{N_{j}}(\norm{BK}+\|\sum_{i\in\mathcal{V}_{j}}G_{i,{j}}C_i/N_{j}\|).$
\end{theorem}
\begin{proof}
	See   Appendix \ref{pf_thm_bounds}.
\end{proof}

Conditions \eqref{L_bound} and \eqref{eq_rho_gamma} are satisfied as long as $L_{j}$ is large enough. Parameters $\rho_{j},\tau_j,\gamma_j$ are introduced for technical reasons. To obtain   parameters   $\{L_{j},\rho_{j},\tau_j,\gamma_j\}$ satisfying \eqref{L_bound} and \eqref{eq_rho_gamma}, one can first find sufficiently small $\rho_j$ and $\tau_j$ such that the first condition in  \eqref{eq_rho_gamma} holds. Then given the selected $\rho_j$ and $\tau_j$,   choose sufficiently large $\gamma_j$ such that the second condition in  \eqref{eq_rho_gamma} holds. Finally, given the chosen $\rho_j$, $\tau_j$, and $\gamma_j$,   find a sufficiently large $L_j$ such that \eqref{L_bound} holds. When there is no uncertainty in  system dynamics and measurements, the estimation and tracking performance of the architecture is illustrated in the following corollary, whose proof follows from Theorem~\ref{thm_error_bound}.
\begin{corollary}
	Under the same conditions as in Theorem~\ref{thm_error_bound}, if the system is uncertainty-free, i.e., $q_w=q_v=0,$ the following result holds
	\begin{align*}
	\lim\limits_{t\rightarrow\infty}	\norm{\hat x_i(t)-x(t)}=0,   \quad 
	\lim\limits_{t\rightarrow\infty}\norm{ x(t)-r(t)}=0,  
	\end{align*}
	provided that the task is not switched and the reference trajectory data is sufficiently large.
\end{corollary}

From Theorem~\ref{thm_error_bound}, a  design of $\{h_e(t)\}$ and $\{h_c(t)\}$ for task-switching condition \eqref{evet_condi} is provided in the following corollary.
\begin{corollary}\label{coro_h}
	For task $\eta$ with $t\in [t_{\eta-1}+1,t_{\eta}]$, let 
	 \begin{align*}
h_e(t)=\norm{D_{\eta}}\bar h_e(t), \quad h_c(t)=\norm{D_{\eta}}\bar h_c(t),
	 \end{align*}
 then under the same conditions as in Theorem~\ref{thm_error_bound}   sequences $\{h_e(t)\}_{t=t_{\eta-1}+1}^{t_{\eta}}$ and $\{h_c(t)\}_{t=t_{\eta-1}+1}^{t_{\eta}}$ satisfy the requirement in Lemma~\ref{lem_event}.
\end{corollary}

Under Assumption~\ref{ass_switching}, the robot is able to compute  $\{h_e(t)\}_{t=t_{\eta-1}+1}^{t_{\eta}}$ and $\{h_c(t)\}_{t=t_{\eta-1}+1}^{t_{\eta}}$ online such that they can be used in Algorithm~\ref{alg:B} for the event-triggered task switching.
From Remark~\ref{rem_random},  sequences $\{h_e(t)\}_{t=t_{\eta-1}+1}^{t_{\eta}}$ and $\{h_c(t)\}_{t=t_{\eta-1}+1}^{t_{\eta}}$ are also random, meaning that the sequences could respectively have different values and lengths in different realizations.

%\subsection{Control performance}

\subsection{Conditions of task accomplishment}\label{sub_time}
In the following theorem we provide a sufficient condition to ensure the accomplishment of a task. Moreover,  we study the  running time for each task.
\begin{theorem}\label{thm_control}
	Suppose task $\eta-1$ is accomplished, then consider  Algorithms~\ref{alg:B}--\ref{alg:bounds} for task $\eta$. Under the same conditions as in Corollary~\ref{coro_h}, if the length of the reference trajectory for task $\eta$ is sufficiently large with $\lim\limits_{l\rightarrow\infty}\norm{D_{\eta}r_{\eta-1,\eta}(l)-c_{\eta}}=0$, and 
	\begin{align}\label{condition_lim}
	\small
	R_{\eta}>\frac{\sqrt{2\beta\lambda_{\max}(P_*)}\left(\norm{BK}\sqrt{\tilde b_{\eta}}+q_w\right)}{\sqrt{\lambda_{\min}(P_*)}\norm{D_{\eta}}}
	\end{align}
	then task $\eta$ will be accomplished for sure, where 
	$\tilde b_{\eta}=3q_{\eta}\lambda_{\max}(P_{\eta})/\lambda_{\min}(P_{\eta}).$
	Furthermore, the running time for task $\eta$ is upper bounded by $\Delta t$, where 
	\begin{align}\label{optim_11}
	%\begin{split}
	\min&\qquad\qquad  \Delta t\\
	\text{s.t.}&   \qquad\Delta t \geq \mathbb{T}_{\eta} \nonumber\\
	&  \max\limits_{l\in[\Delta t-\mathbb{T}_{\eta},\Delta t)} \left\{ \phi(\eta,l)\right\}\leq R_{\eta},\nonumber
	\end{align}
	where $\phi(\eta,l)=h_c(t_{\eta-1}+l)+\norm{D_{\eta}r_{\eta-1,\eta}(l)-c_{\eta}}.$
\end{theorem}
\begin{proof}
	%	See Appendix \ref{pf_thm_control}.
	We make a conjecture that  task $\eta$ can not be accomplished in finite time, i.e., $t_{\eta}=\infty$.
	It holds that 
$\lim\limits_{t\rightarrow\infty}\bar h_e(t)=\sqrt{\tilde b_{\eta}}.$
Then we derive that
	$$\limsup\limits_{t\rightarrow\infty}\bar h_c(t)\leq  \sqrt{\frac{2\beta\lambda_{\max}(P_*)}{\lambda_{\min}(P_*)}}\left(\norm{BK}\sqrt{\hat b_{\eta}}+q_w\right).$$
		Due to $\lim\limits_{t\rightarrow\infty}\norm{D_{\eta}r(t)-c_{\eta}}=0$, 	
it follows from  Corollary~\ref{coro_h} and Lemma~\ref{lem_event} that 
	$f(t)$  tends to zero as $t$ goes to infinity. In other words, $f(t)$ would be sufficiently small provided with a sufficiently large time $t$. Therefore, 		if \eqref{condition_lim} is satisfied, 
	 there is a finite time $T_*$, such that $f(t)<R_{\eta}$ for $t\geq T_*$.    This contradicts   the conjecture, thus task $\eta$ will be accomplished in finite time for sure.
	Note that the second constraint of optimization problem \eqref{optim_11} ensures that $\Delta t$ is an upper bound of the running time for task $\eta$. Moreover, the provided conditions ensure the feasibility of optimization problem \eqref{optim_11}.  
	%
	%To ensure that the state $x(t)$ stays in the ball at least $\mathbb{T}_\eta$ successive time instants, we are supposed to prove that at the time $t_{\eta+1}-\mathbb{T}_{\eta+1}$, the state $x(t)$ has been in the ball for sure. By Algorithm~\ref{alg:B}, it  holds that   $\forall t\in [T_{\eta,\eta+1}-\mathbb{T}_{\eta+1},T_{\eta,\eta+1})$,
	%\begin{align}\label{trigger_cond}
	%f(t_{\eta}+\Delta t)\leq R_{\eta+1}, 
	%\end{align}
	%We consider $\forall t\in [t_{\eta}+T_{\eta,\eta+1}-\mathbb{T}_{\eta+1},t_{\eta}+T_{\eta,\eta+1})$,
	%\begin{align*}
	%&h_c(t)+\norm{r(t)-c_{\eta+1}}\\
	%\leq& \max_{l\in [T_{\eta,\eta+1}-\mathbb{T}_{\eta+1},T_{\eta,\eta+1}) }\left\{h_c(t_{\eta}+l)+\norm{r_{\eta,\eta+1}(l)-c_{\eta+1}}\right \},\\
	%\leq& \bar h_{c,\eta}+D_{\eta}
	%\end{align*}
	%where $\bar h_{c,\eta},D_{\eta}$ are given in \eqref{bound_control}.
	%Then \eqref{trigger_cond} is met  if \eqref{eq_condition} holds.
	%The condition \eqref{eq_condition} ensures the state stays in the ball at the last $\mathbb{T}_\eta$ time instants. Under the condition, it is feasible to obtain an upper bound of the switching time   by solving the optimization problem in \eqref{optim_11}.
	
\end{proof}

	Condition \eqref{condition_lim} is satisfied if uncertainty  bounds $q_w$ and $q_v$ are sufficiently small. For example, in the case $q_w=q_v=0$,   we have $\tilde b_{\eta}=0$. As a result,  condition \eqref{condition_lim} is satisfied if and only if $R_{\eta}>0$.
	  Theorem~\ref{thm_control} shows that it is feasible to obtain an upper bound of the task running time against system uncertainties, and  the initial estimation error and tracking deviation. This bound can be used  to evaluate the efficiency of the architecture and the generation of reference trajectories. Although   optimization problem   \eqref{optim_11}  could be non-convex, they can be solved offline by using some existing algorithms, such as enumeration methods or heuristic optimization algorithms.

\section{Numerical Simulations}\label{sec:simulation}
In this section, we revisit the motivating example in Fig.~\ref{fig:robot} with the proposed architecture. 
%i.e., $x(t)=[x_1(t),v_1(t),x_2(t),v_2(t)]^{\sf T}$, 
The robot's state is four-dimensional, and it consists of  longitude and latitude positions and the corresponding velocities. The parameter matrices in \eqref{sys} are assumed to be
$A=\left(\begin{smallmatrix}
1&T&0&0\\
0&1&0&0\\
0&0&1&T\\
0&0&0&1
\end{smallmatrix}\right)$, and $B=\left(\begin{smallmatrix}
0&0\\
T&0\\
0&0\\
0&T
\end{smallmatrix}\right)$, where $T=0.01$.
Each active sensor $i=1,2,\dots,12$ in  Fig.~\ref{fig:robot} (b)--(d) is assumed to have the following measurement matrix $C_i=(1,0,0,0) \text{ if $i$ is odd}$, otherwise $C_i=(0,0,1,0)$.
Assume  
each element of the process uncertainty follows a uniform distribution in [-0.005.0.005]. 
The measurement uncertainty follows a uniform distribution in [-0.01.0.01]. The initial state of the robot $x_0=(100,-0.1,20,-0.06)^{\sf T}$. 
 Each element of the initial estimation error of each sensor follows a uniform distribution in [-5,5].  
The centers of the three targeted sets in Fig.~\ref{fig:robot} (a) are $c_1=(50, 5)^{\sf T},c_2=(20, 50)^{\sf T},c_3=(80, 60)^{\sf T}$ with radius $ 
R_1=5,R_2=10$, and $R_3=15$ respectively. 
The sets are used to represent the desired positions.
 With these parameters and constraint \eqref{eq_reference}, we use cubic spline interpolation to generate three reference trajectories  connecting the centers of the three targeted sets, respectively. The lengths of the three trajectories are 
501, 301, and 601 respectively.

Given each active DSN in Fig.~\ref{fig:robot} (b)--(d),
choose   parameters $\{G_{i,\eta}\}$ and $K$, $\eta=1,2,3$, for \text{Algorithms~\ref{alg:B}--\ref{alg:A}} as in Lemma~\ref{lem_parameter}, such that the eigenvalues of $A-G_{\eta}\tilde C_{\eta}$ and $A+BK$ are placed at $(-0.2,-0.1,0.1,0.2)^{\sf T}$ and  $(-0.1,-0.1,0.1,0.1)^{\sf T}$, respectively. Choose parameter $\alpha_{\eta}=2/(\lambda_{2}(\mathcal{L}_{\eta})+\lambda_{\max}(\mathcal{L}_{\eta}))$. The desired dwelling time for each targeted set is $\mathbb{T}_{\eta}=2$. Communication parameter $L_{\eta}=5.$
The design of $\{h_e(t)\}$ and $\{h_c(t)\}$ is as in Corollary~\ref{coro_h} and Algorithm~\ref{alg:bounds}.

We conduct a Monte Carlo experiment with $100$ runs. To evaluate the performance, we consider   maximum state estimation error $\mu(t)$ and   maximum trajectory tracking deviation $\tau(t)$ respectively defined as follows
\begin{align}\label{simu_notation_error}
\begin{split}
%\small
\mu(t)&=\max_{j=1,2,\dots,100}\max_{i\in \mathcal{V}_{\eta}}\norm{\hat x_{i}^j(t)-x^j(t)},\\
\tau(t)&=\max_{j=1,2,\dots,100}\norm{x^j(t)-r(t)},\\
\end{split}
\end{align}
where $x^j(t)$ and $\hat x_{i}^j(t)$ are the robot state and  its estimate by sensor $i$   at time $t$ in the $j$-th run, respectively.
Under the above setting, we run  Algorithms~\ref{alg:B}--\ref{alg:A} and obtain \text{Figs.~\ref{fig:est_and_track_error}--\ref{fig:triggering_times}}.
Fig.~\ref{fig:est_and_track_error} shows the dynamics of  state estimation error $\mu(t)$ and trajectory tracking deviation $\tau(t)$. It shows that $\tau(t)$ becomes stable after transience and $\tau(t)$ is stable for each task except the sharp increase from task switching. We randomly choose one realization to show the performance of the robot in tracking the reference trajectory (position) in Fig.~\ref{fig:balls_tracking}. It shows the robot is able to reach each targeted set successfully and switch to the next task sequentially. In Fig.~\ref{fig:triggering_times}, the dynamics of  event-triggered parameter $f(t)$ in Algorithm~\ref{alg:B} and the task-switching  points are shown. The decrease of $f(t)$ is due to the performance improvement for state estimation and trajectory tracking  as time goes on, while the fluctuation of $f(t)$ is resulted from   task switching.

\begin{figure}[t]
	\centering
	\includegraphics[scale=0.48]{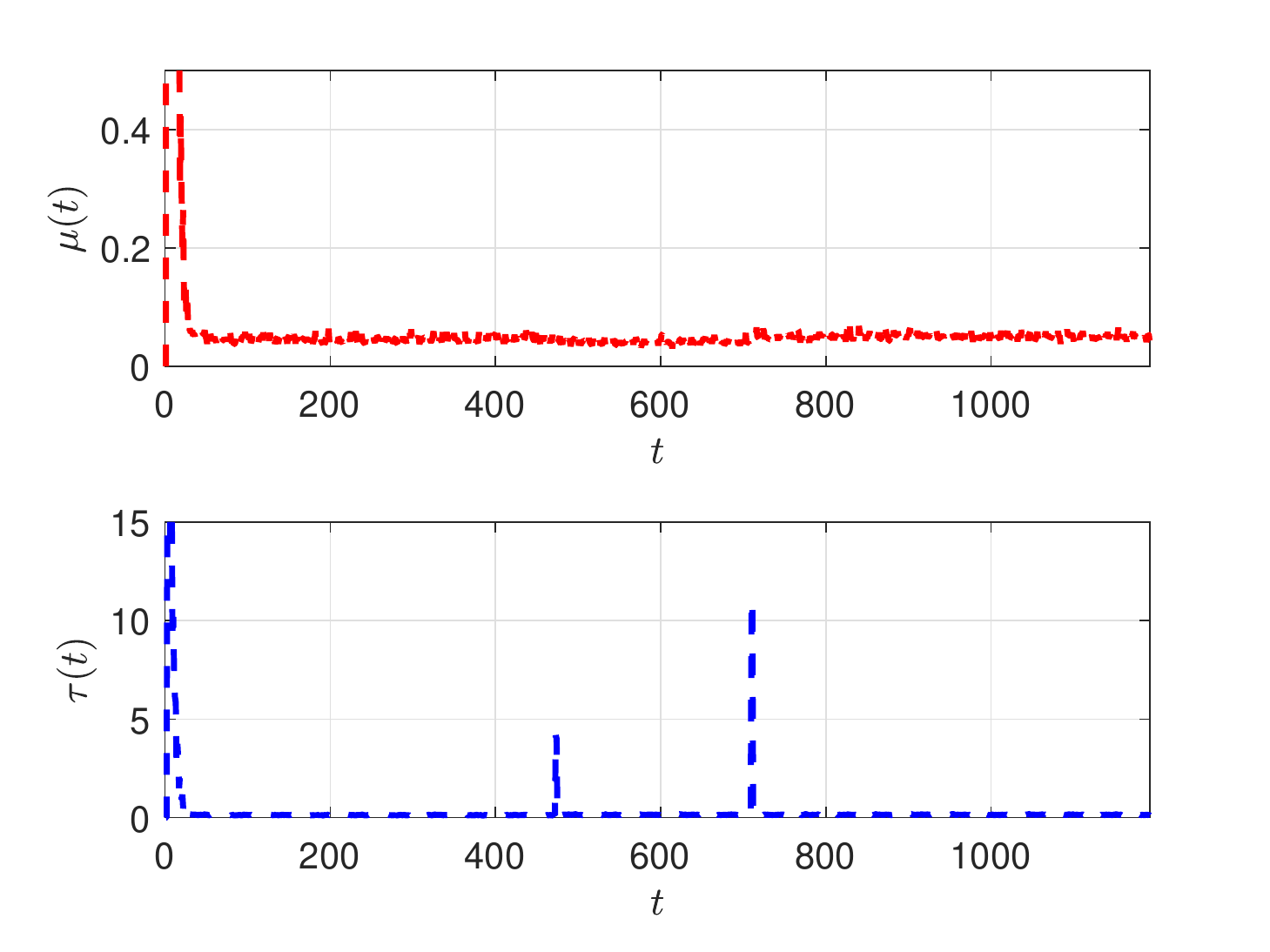}
	\caption{The state estimation error $\mu(t)$ and trajectory tracking deviation $\tau(t)$ of Algorithms~\ref{alg:B}--\ref{alg:A}. The two sharp increases of $\tau(t)$ are resulted from   task switching.  }
		\label{fig:est_and_track_error}
\end{figure}

\begin{figure}[t]
	\centering
	\includegraphics[scale=0.48]{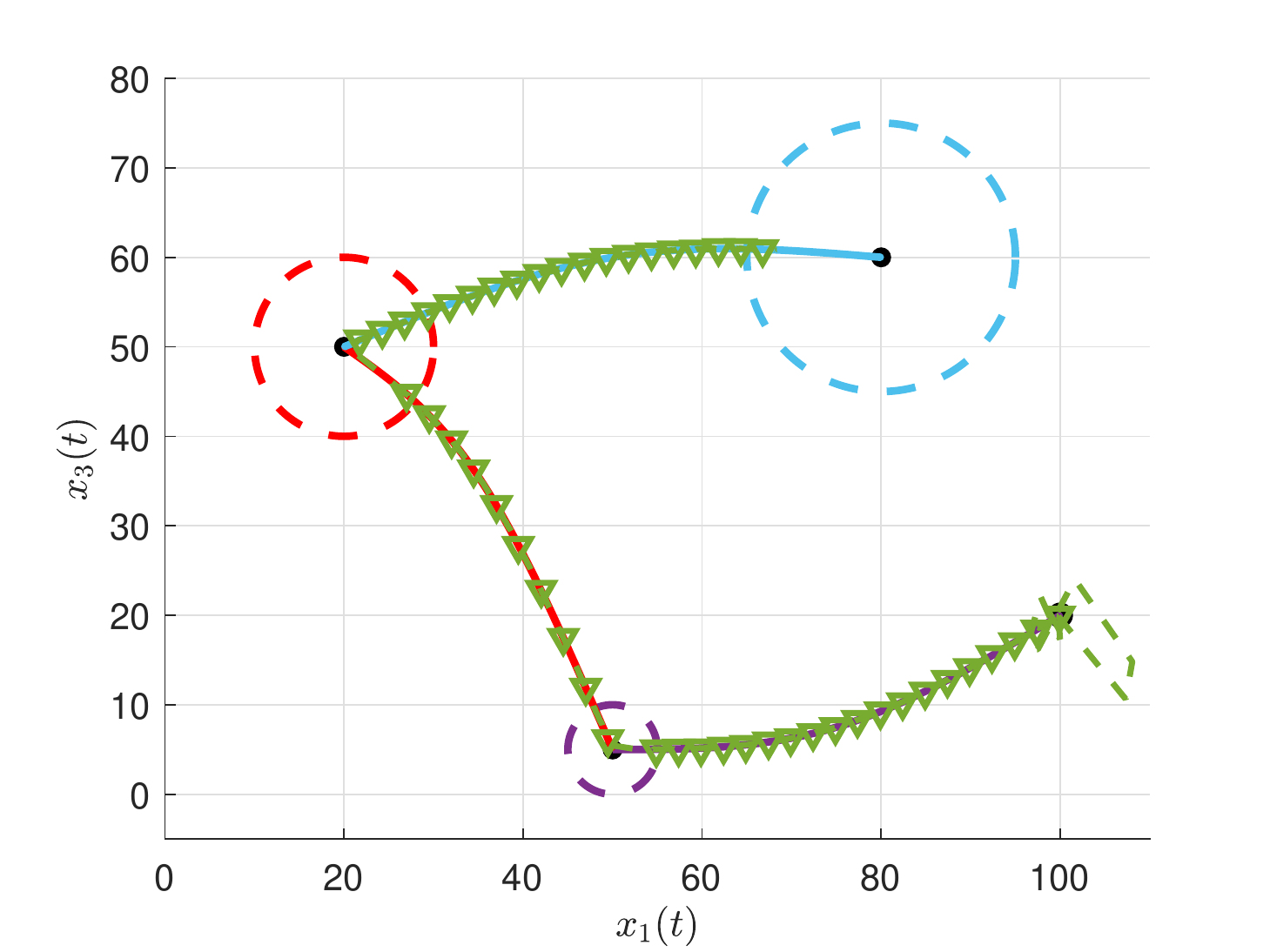}
	\caption{Reference position trajectory  tracking and task switching of Algorithms~\ref{alg:B}--\ref{alg:A}. In the transient time period, since the estimate is not accurate, the estimate-based control input does not drive the robot  to the reference position trajectory. As time goes on, the robot  is able to converge to the reference trajectory and accomplish the three tasks sequentially.}
		\label{fig:balls_tracking}
\end{figure}

\begin{figure}[t]
	\centering
	\includegraphics[scale=0.48]{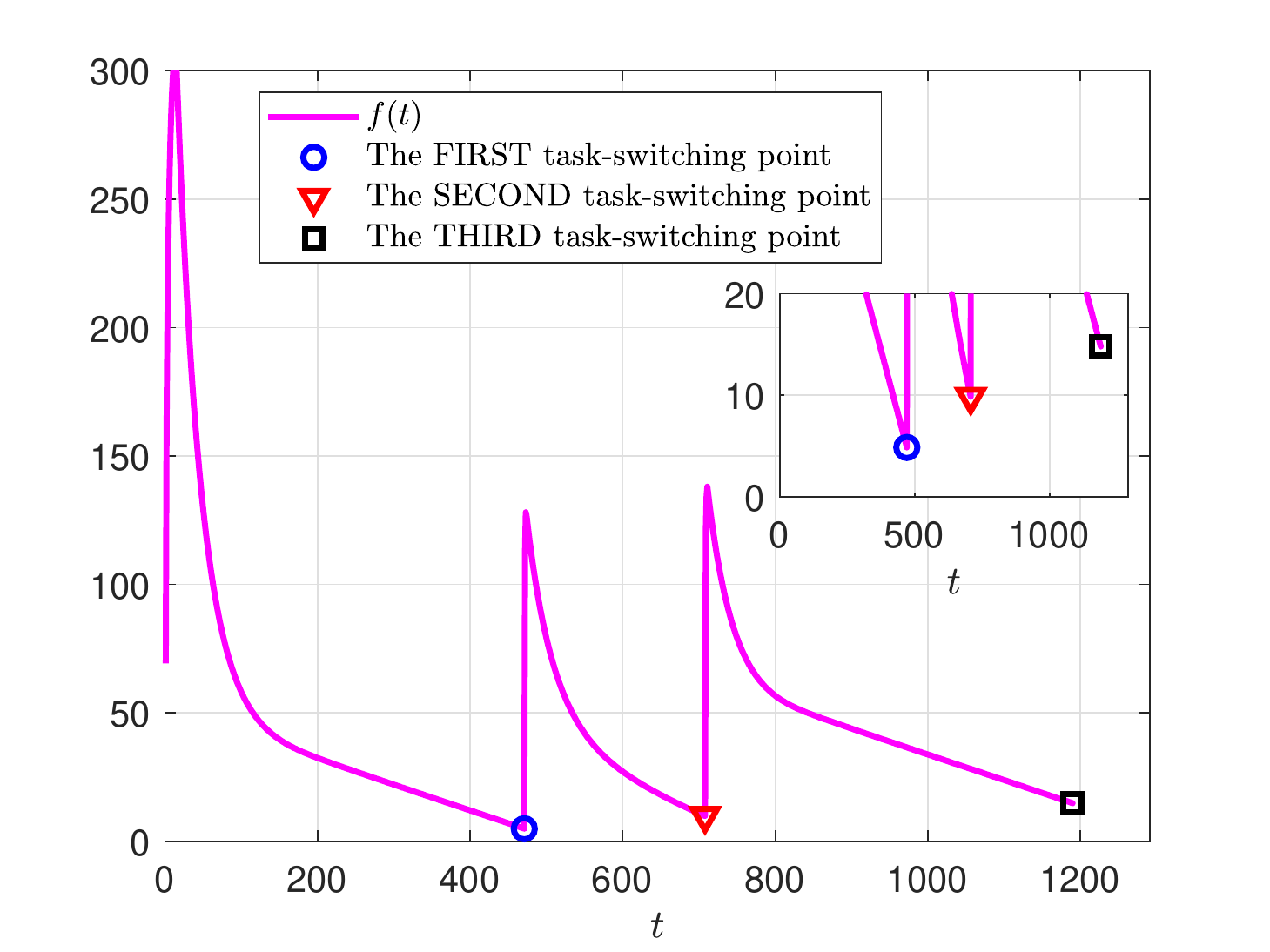}
\caption{The dynamics of event-triggered parameter  $f(t)$ and the task-switching  points. In each task, parameter $f(t)$ decreases until the current task is accomplished. From one task to another, the sharp increase of $f(t)$ is due to the switching of the reference trajectories corresponding to the two sharp increases of $\tau(t)$ in Fig.~\ref{fig:est_and_track_error}. }
\label{fig:triggering_times}
\end{figure}

\begin{figure}[t]
	\centering
	\includegraphics[scale=0.48]{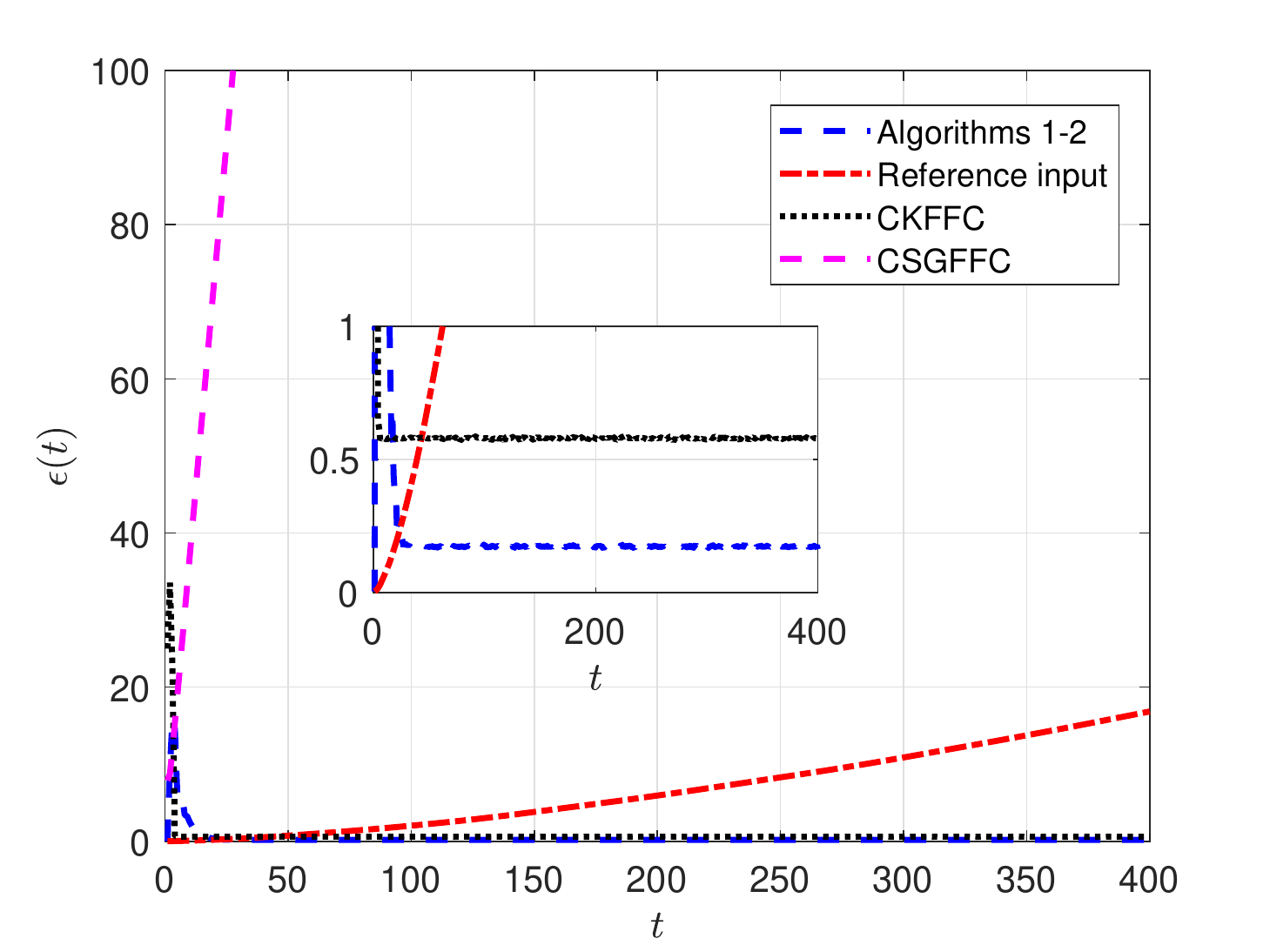}
	\caption{Comparison of navigation errors  of several algorithms for task one. 
%		Both  CSGFFC and the reference input do not stabilize the robot for navigation. Algorithms~\ref{alg:B}--\ref{alg:A} and CKFFC  are able to obtain stable navigation error (trajectory tracking deviation). 
%		Meanwhile, Algorithms~\ref{alg:B}--\ref{alg:A} achieves better performance in the given task.  
		}
	\label{fig:compare}
\end{figure}

Next, we consider task one within time interval [1,400]  under the same setting as in the first case to  compare the navigation performance of Algorithms~\ref{alg:B}--\ref{alg:A} with  some other algorithms. The centralized Kalman filter (CKF) and the CSGE   \cite{khan2014collaborative}   are considered here, and they are equipped with our feedback controller in Algorithm~\ref{alg:B} and then abbreviated by   CKFFC and   CSGEFC, respectively. Moreover, we   study   navigation performance by using the reference input.
We compare these algorithms based on the   average navigation error (trajectory-tracking deviation):
$\epsilon(t)=\frac{1}{100}\sum_{j=1}^{100}\norm{x^j(t)-r(t)}.$
The comparison result is provided in 
Fig.~\ref{fig:compare}, which   shows that the navigation errors of  Algorithms~\ref{alg:B}--\ref{alg:A} and CKFFC are both tending to small neighborhoods of zero. Meanwhile, Algorithms~\ref{alg:B}--\ref{alg:A} outperforms the other three algorithms  in the considered scenario. Note that the CKF is not optimal in this case, since the uncertainties here follow a uniform distribution but not a normal distribution.
Moreover, Fig.~\ref{fig:compare} illustrates that   feedback control outperforms the reference control  for  trajectory tracking in  uncertain environments.

\section{Conclusions}\label{sec:conclusion}
We  studied  how to    navigate a mobile robot over active distributed  sensor networks such that the robot state is able to reach multiple targeted sets in order and stays inside no less than the desired time  by  following  preset reference  trajectories. 
We proposed a new task-switching navigation architecture 
consisting of a task-switching feedback controller for the robot and a two-time-scale distributed estimator for each active sensor. 
With the controller, the robot is able to accomplish a task by following a reference trajectory and switch to the next task when an event-triggered condition is fulfilled.  
With the estimator, each active sensor is able to estimate the robot state. 
We found two time-varying sequences to bound the  state estimation error and trajectory tracking deviation in the architecture, respectively. In addition, we showed that the two sequences  play an essential role in the event-triggered scheme.
Moreover, a sufficient condition for    the accomplishment of each task was established and 
an upper bound of 
the running time of each task was obtained.

\appendix 

\subsection{Useful lemmas}
In this subsection, some lemmas and their proofs are provided. The lemmas are useful for the proof of Theorem~\ref{thm_error_bound}.
Recall  that for task $\eta$, the network $\mathcal{G}_{\eta}=(\mathcal{V}_{\eta},\mathcal{E}_{\eta},\mathcal{A}_{\eta})$ with $|\mathcal{V}_{\eta}|=N_{\eta}$ is active and sensor $s_{\eta}\in \mathcal{V}_{\eta}$ is the sensor sharing estimates with the robot, where $\eta\in \{1,2,\dots,Q_s\}$.  
For task $\eta$ at time $t$,  denote
\begin{align}\label{eq_notations1}
%\begin{split}
e_i(t)&=\hat x_i(t)-x(t)\in\mathbb{R}^n,\nonumber\\
\hat x_{\avg}(t)&=\sum_{i\in\mathcal{G}_{\eta}}\hat x_{i}(t)/N_{\eta}\in\mathbb{R}^{n\times 1},\nonumber\\
\hat X(t)&=[\hat x_{i}(t)]_{i\in\mathcal{G}_{\eta}}\in\mathbb{R}^{nN_{\eta}\times 1},\nonumber \\
\mathbf{e}_{\avg}(t)&=\hat X(t)-\textbf{1}_{N_{\eta}}\otimes\hat x_{\avg}(t)\in\mathbb{R}^{nN_{\eta}\times 1}, \nonumber\\
\mathbf{e}_{\net}(t)&=\hat x_{\avg}(t)-x(t)\in\mathbb{R}^{n\times 1},  \nonumber\\
E(t)&=[ e_i(t)]_{i\in\mathcal{G}_{\eta}}\in\mathbb{R}^{nN_{\eta}\times 1},\\
Y(t)&=[y_i(t)]_{i\in\mathcal{G}_{\eta}}\in\mathbb{R}^{(\sum_{i\in\mathcal{G}_{\eta}}M_i)\times 1},\nonumber\\
V(t)&=[v_i(t)]_{i\in\mathcal{G}_{\eta}}\in\mathbb{R}^{(\sum_{i\in\mathcal{G}_{\eta}}M_i)\times 1}.\nonumber
%\bar I_{k}&=(\textbf{1}_{N_{\eta}}\otimes  I_n)(\textbf{1}_{N_{\eta}}\otimes  I_n)^{\sf T}/N_{\eta}\in\mathbb{R}^{nN_{\eta}\times nN_{\eta}}, \nonumber\\
%\mathbf{G}_{\eta}&=\diag\{G_i\}_{i\in\mathcal{G}_{\eta}}\in\mathbb{R}^{nN_{\eta}\times(\sum_{i\in\mathcal{G}_{\eta}}M_i) }, \nonumber \\
%\mathbf{C}_{\eta}&=\diag\{C_i\}_{i\in\mathcal{G}_{\eta}}\in\mathbb{R}^{(\sum_{i\in\mathcal{G}_{\eta}}M_i)\times nN_{\eta}}\nonumber\\
%\mathbf{L}_{\eta}&=I-\alpha_{\eta}(\mathcal{L}_{\eta}\otimes I_n)\nonumber\\
%\mathbf{A}_{\eta}&=I_{N_{\eta}}\otimes (A+BK).\nonumber
%\end{split}
\end{align}

\begin{lemma}\label{lem_equalties}
The following equalities hold:
\begin{enumerate}
	\item $\bar I_{\eta}(I_n\otimes A)=(I_n\otimes A)\bar I_{\eta}$
	\item $E(t)=\mathbf{e}_{\avg}(t)+\textbf{1}_{N_{\eta}}\otimes\mathbf{e}_{\net}(t)$
	\item $(I-\bar I_{\eta})(I_{N_{\eta}}\otimes x)=\bold{0} \text{ for any } x\in\mathbb{R}^n$
	\item 
	$(I-\bar I_{\eta})\mathbf{L}_{\eta}^{L_{\eta}}=\mathbf{L}_{\eta}^{L_{\eta}}(I-\bar I_{\eta})$
	\item $(\textbf{1}_{N_{\eta}}^{\sf T}\otimes  I_n)\bar I_{\eta}=(\textbf{1}_{N_{\eta}}^{\sf T}\otimes  I_n)\mathbf{L}_{\eta}^{L_{\eta}}=(\textbf{1}_{N_{\eta}}^{\sf T}\otimes  I_n)$
	\item $\left(\mathbf{L}_{\eta}-\bar I_{\eta}\right)^{L_{\eta}}\mathbf{A}_{\eta}\mathbf{e}_{\avg}(t)=\mathbf{L}_{\eta}^{L_{\eta}}\mathbf{A}_{\eta}\mathbf{e}_{\avg}(t),$
\end{enumerate}
where $L_{\eta}$ is the communication parameter in Algorithm~\ref{alg:A}, and
\begin{align}\label{eq_notations2}
%\begin{split}
%e_i(t)&=\hat x_i(t)-x(t),\nonumber\\
%\hat x_{\avg}(t)&=\sum_{i\in\mathcal{G}_{\eta}}\hat x_{i}(t)/N_{\eta}\in\mathbb{R}^{n\times 1},\nonumber\\
%\hat X(t)&=[\hat x_{i}(t)]_{i\in\mathcal{G}_{\eta}}\in\mathbb{R}^{nN_{\eta}\times 1},\nonumber \\
%\mathbf{e}_{\avg}(t)&=\hat X(t)-\textbf{1}_{N_{\eta}}\otimes\hat x_{\avg}(t)\in\mathbb{R}^{nN_{\eta}\times 1}, \nonumber\\
%\mathbf{e}_{\net}(t)&=\hat x_{\avg}(t)-x(t)\in\mathbb{R}^{n\times 1},  \nonumber\\
%E(t)&=[ e_i(t)]_{i\in\mathcal{G}_{\eta}}\in\mathbb{R}^{nN_{\eta}\times 1},\\
%Y(t)&=[y_i(t)]_{i\in\mathcal{G}_{\eta}}\in\mathbb{R}^{(\sum_{i\in\mathcal{G}_{\eta}}M_i)\times 1},\nonumber\\
%V(t)&=[v_i(t)]_{i\in\mathcal{G}_{\eta}}\in\mathbb{R}^{(\sum_{i\in\mathcal{G}_{\eta}}M_i)\times 1},\nonumber\\
\bar I_{k}&=(\textbf{1}_{N_{\eta}}\otimes  I_n)(\textbf{1}_{N_{\eta}}\otimes  I_n)^{\sf T}/N_{\eta}\in\mathbb{R}^{nN_{\eta}\times nN_{\eta}}, \nonumber\\
\mathbf{G}_{\eta}&=\diag\{G_i\}_{i\in\mathcal{G}_{\eta}}\in\mathbb{R}^{nN_{\eta}\times(\sum_{i\in\mathcal{G}_{\eta}}M_i) },  \\
\mathbf{C}_{\eta}&=\diag\{C_i\}_{i\in\mathcal{G}_{\eta}}\in\mathbb{R}^{(\sum_{i\in\mathcal{G}_{\eta}}M_i)\times nN_{\eta}}\nonumber\\
\mathbf{L}_{\eta}&=I-\alpha_{\eta}(\mathcal{L}_{\eta}\otimes I_n)\in\mathbb{R}^{nN_{\eta}\times nN_{\eta}}\nonumber\\
\mathbf{A}_{\eta}&=I_{N_{\eta}}\otimes (A+BK)\in\mathbb{R}^{nN_{\eta}\times nN_{\eta}}.\nonumber
%\end{split}
\end{align}
%\begin{itemize}
%	\item $\bar I_{\eta}(I_n\otimes A)=(I_n\otimes A)\bar I_k$
%	\item $(I-\bar I_k)\left(I-\alpha_k(\mathcal{L}_k\otimes I_n)\right)^{L_{k}}=\left(I-\alpha_k(\mathcal{L}_k\otimes I_n)\right)^{L_{k}}(I-\bar I_k)$
%	\item $\left(I-\alpha_k(\mathcal{L}_k\otimes I_n)-\bar I_{k}\right)^{L_{k}}(I_{N_k}\otimes (A+BK))\mathbf{e}_{\avg}(t)=\left(I-\alpha_k(\mathcal{L}_k\otimes I_n)\right)^{L_{k}}(I_{N_k}\otimes (A+BK))\mathbf{e}_{\avg}(t)$
%	\item $(I-\bar I_k)(I_N\otimes x)=0$ for any $x\in\mathbb{R}^n$.
%	\item $E(t)=\mathbf{e}_{\avg}(t)+\textbf{1}_{N_k}\otimes\mathbf{e}_{\net}(t).$
%\end{itemize}
\end{lemma}
\begin{proof}
	The proof of the last equality is similar to that of Lemma~1.2 in \cite{He2020design}. The rest of equalities are straightforward  derived by noting that $\mathcal{L}_{\eta}$ is a symmetric matrix with row (resp. column) sum equal to zero.	
	\end{proof}

Note that  $\mathbf{e}_{\net}(t)$ in \eqref{eq_notations1} stands for the error between the average estimate of all sensors and the true state, while $\mathbf{e}_{\avg}(t)$ represents the  difference between the estimate of each active sensor and the average estimate.  
With the notations in \eqref{eq_notations1}, given  $\gamma_{\eta}>0$, denote
$$\mathbf{e}(t):=(\textbf{1}_{N_{\eta}}^{\sf T}\otimes\mathbf{e}_{\net}^{\sf T}(t),\gamma_{\eta}\mathbf{e}_{\avg}^{\sf T}(t))^{\sf T}\in\mathbb{R}^{nN_{\eta}\times nN_{\eta}}.$$ 
Instead of analyzing the estimation error $e_i(t)$ directly, 
 we aim to analyze $\mathbf{e}(t)$, whose  dynamics are given in the following lemma.
\begin{lemma}\label{error_dynamic}
	Suppose    network  $\mathcal{G}_{\eta}$   is active at time $t$, then $\mathbf{e}(t)$   satisfies 
	\begin{align}
	\mathbf{e}(t)=&\begin{pmatrix}
	M_{1,1}(\eta)&M_{1,2}(\eta)\\
	M_{2,1}(\eta)&M_{2,2}(\eta)
	\end{pmatrix}\mathbf{e}(t-1)+\begin{pmatrix}
	W_{1,1}(t)\\
	W_{1,2}(t)
	\end{pmatrix},\label{lem_extend}
	\end{align}
	where 
	\begin{align}\label{eq_notation_error1}
	%\begin{split}
	M_{1,1}(\eta)&=I_{N_{\eta}}\otimes\big(A- \sum_{i\in\mathcal{G}_{\eta}}G_{i,k}C_{i}/N_{\eta}\big),\nonumber \\ M_{1,2}(\eta)&=\textbf{1}_{N_{\eta}}\otimes \big(BK(m_{s_{\eta}}\otimes I_n)\nonumber\\
	&\quad+  (\textbf{1}_{N_{\eta}}^{\sf T}\otimes I_n)\mathbf{G}_{\eta}\mathbf{C}_{\eta}/N_{\eta}\big)/\gamma_{\eta},  \\ M_{2,1}(\eta)&=-\gamma_{\eta}\mathbf{L}_{\eta}^{L_{\eta}}(I-\bar I_{\eta})\mathbf{G}_{\eta}\mathbf{C}_{\eta}, \nonumber\\
	M_{2,2}(\eta)&=\left(\mathbf{L}_{\eta}-\bar I_{\eta}\right)^{L_{\eta}}\mathbf{A}_{\eta}+M_{2,1}(\eta)/\gamma_{\eta}, \nonumber\\
	W_{1,1}(t)&=\textbf{1}_{N_{\eta}}\otimes \left( (\textbf{1}_{N_{\eta}}^{\sf T}\otimes I_n)\mathbf{G}_{\eta}V(t)/N_{\eta}-w(t)\right),\nonumber\\	
	W_{1,2}(t)&=\gamma_{\eta}\mathbf{L}_{\eta}^{L_{\eta}}	(I-\bar I_{\eta})\mathbf{G}_{\eta}V(t),\nonumber
	\end{align}
in which $m_{s_{\eta}}\in \mathbb{R}^{1 \times N_{\eta}}$ has all zero elements but one in the $s_{\eta}$-th element. 
\end{lemma}
\begin{proof}
	%	See Appendix \ref{pf_error_dynamic}.
With the equalities in Lemma~\ref{lem_equalties}, we derive the dynamics of $\gamma_{\eta}\mathbf{e}_{\avg}(t)$ as follows
	\begin{align}\label{eq_avg_error}
	&\gamma_{\eta}\mathbf{e}_{\avg}(t+1)\\
	=&M_{2,2}(\eta)\gamma_{\eta}\mathbf{e}_{\avg}(t)+M_{2,1}(\eta)(\textbf{1}_{N_{\eta}}\otimes\mathbf{e}_{\net}(t))+W_{1,2}(t).\nonumber
	\end{align}	
Moreover, we derive that
%	\begin{align}\label{eq_recursive_net}
%	&\mathbf{e}_{\net}(t+1)\\
%	=&(A+BK)\mathbf{e}_{\net}(t)-w(t)-BK(m_{s_{\eta}}\otimes I_n)E(t)\nonumber\\
%	&+\frac{1}{N_{\eta}} (\textbf{1}_{N_{\eta}}^{\sf T}\otimes I_n) \bar I_{\eta}\mathbf{L}_{\eta}^{L_{\eta}}\mathbf{G}_{\eta}(Y(t)-\mathbf{C}_{\eta}\hat X(t)).\nonumber
%	\end{align}	
% Substituting $E(t)$ from Lemma~\ref{lem_equalties} into \eqref{eq_recursive_net} yields 
	\begin{align}\label{eq_recursive_net2}
	&\mathbf{e}_{\net}(t+1)\\
	=&A\mathbf{e}_{\net}(t)-w(t)-BK(m_{s_{\eta}}\otimes I_n)\mathbf{e}_{\avg}(t)\nonumber\\
	&+\frac{1}{N_{\eta}} (\textbf{1}_{N_{\eta}}^{\sf T}\otimes I_n) \bar I_{\eta}\mathbf{L}_{\eta}^{L_{\eta}}\mathbf{G}_{\eta}(Y(t)-\mathbf{C}_{\eta}\hat X(t))\nonumber\\
	=&\left(A-\frac{1}{N_{\eta}}\sum_{i\in\mathcal{G}_{\eta}}G_{i,k}C_{i}\right)\mathbf{e}_{\net}(t)\nonumber\\
	&+\frac{1}{N_{\eta}} (\textbf{1}_{N_{\eta}}^{\sf T}\otimes I_n)\mathbf{G}_{\eta}V(t)-w(t)\nonumber\\
	&-\bigg(BK(m_{s_{\eta}}\otimes I_n)+\frac{1}{N_{\eta}} (\textbf{1}_{N_{\eta}}^{\sf T}\otimes I_n)\mathbf{G}_{\eta}\mathbf{C}_{\eta}\bigg)\mathbf{e}_{\avg}(t),\nonumber
	\end{align}	
	where the last equality is obtained from the second and the fourth equalities in Lemma~\ref{lem_equalties}.
	By \eqref{eq_avg_error}--\eqref{eq_recursive_net2},  the conclusion holds.
\end{proof}

%		The scalar $\gamma$ in Lemma \ref{error_dynamic} is introduced   to deal with all possible time-varying communication connections between the robot and the sensor $s_{\eta}\in\{1,2,\dots,N\}$.  	
%		The following two lemmas are provided for subsequent analysis.	
\begin{lemma}\label{lem_inequality}
	Assume $P_1\in\mathbb{R}^{m\times m}$ and $P_2\in\mathbb{R}^{n\times n}$ are positive semi-definite matrices, and $A\in\mathbb{R}^{m\times m}$, $B\in\mathbb{R}^{m\times n}$, $C\in\mathbb{R}^{n\times m}$, $D\in\mathbb{R}^{n\times n}$. Given     $\tau>0$, the following inequality holds
$\left(\begin{smallmatrix}
A&B\\
C&D
\end{smallmatrix}\right)^{\sf T}	\left(\begin{smallmatrix}
P_1&0\\
0&P_2
\end{smallmatrix}\right)	\left(\begin{smallmatrix}
A&B\\
C&D
\end{smallmatrix}\right)\preceq 	\left(\begin{smallmatrix}
\hat P_1&0\\
0&\hat P_2
\end{smallmatrix}\right),$
	where $\hat P_1=(1+\tau)(A^{\sf T}P_1A+C^{\sf T}P_2C)$ and $\hat P_2=(1+1/\tau)(B^{\sf T}P_1B+D^{\sf T}P_2D).$
\end{lemma}
\begin{proof}
	%	See Appendix \ref{pf_lem_inequality}.
	Let $x=(x_1^{\sf T},x_2^{\sf T})^{\sf T}\in\mathbb{R}^{m+n}\neq 0$.  Then
	\begin{align}
	&x^{\sf T}	\begin{pmatrix}
	A&B\\
	C&D
	\end{pmatrix}^{\sf T}	\begin{pmatrix}
	P_1&0\\
	0&P_2
	\end{pmatrix}	\begin{pmatrix}
	A&B\\
	C&D
	\end{pmatrix}x\nonumber\\
	=&x_1^{\sf T}(A^{\sf T}P_1A+C^{\sf T}P_2C)x_1+x_2^{\sf T}(B^{\sf T}P_1B+D^{\sf T}P_2D)x_2\nonumber\\
	&+ x_1^{\sf T}(A^{\sf T}P_1B+C^{\sf T}P_2D)x_2+x_2^{\sf T}(B^{\sf T}P_1A+D^{\sf T}P_2C)x_1\nonumber\\
	\overset{(a)}{\leq} &(1+\tau)x_1^{\sf T}(A^{\sf T}P_1A+C^{\sf T}P_2C)x_1\nonumber\\
	&+(1+\frac{1}{\tau})x_2^{\sf T}(B^{\sf T}P_1B+D^{\sf T}P_2D)x_2,\label{ineq}
	\end{align}
	where $(a)$ holds due to 
	\begin{align*}
	\small
	\left(\sqrt{\tau}Ax_1-\frac{1}{\sqrt{\tau}}Bx_2\right)^{\sf T}P_1\left(\sqrt{\tau}Ax_1-\frac{1}{\sqrt{\tau}}Bx_2\right)&\succeq 0\\
	\left(\sqrt{\tau}Cx_1-\frac{1}{\sqrt{\tau}}Dx_2\right)^{\sf T}P_2\left(\sqrt{\tau}Cx_1-\frac{1}{\sqrt{\tau}}Dx_2\right)&\succeq 0.
	\end{align*}
	The conclusion follows from \eqref{ineq} for any $x\neq 0$.
\end{proof}
\begin{lemma}\label{lem_riccati}\cite{anderson2012optimal}
	If   $F\in\mathbb{R}^{n\times n}$ is Schur stable, then the algebraic  Riccati equation $F^{\sf T}PF+I_n=P$ has a unique finite solution  
	$P=\sum_{i=0}^{\infty}(F^i)^{\sf T}F^i$.
	%									 Moreover, $\lambda_{\min}(P)\geq 1.$
\end{lemma}

%					The matrix $P$ in Lemma \ref{lem_riccati} will be used in  building a Lyapunov function of the error dynamics $\mathbf{e}(t)$ in Lemma~\ref{error_dynamic}.
%					The following lemma provides an explicit design of the communication step $L$.					
\begin{lemma}\label{lem_parameters}
	Under Assumptions \ref{ass_stabilization}--\ref{ass_switching}, if 
  conditions \eqref{L_bound} and \eqref{eq_rho_gamma} are satisfied, then  
	\begin{align*}
	\bar M_{1,1}(\eta)&\preceq -\frac{1}{3}I,\quad 									\bar M_{2,2}(\eta)\preceq -\frac{1}{3}I,
	\end{align*}
	where 
	\begin{align}\label{eq_MP}
	\small
	\bar M_{1,1}(\eta)&=(1+\rho_{\eta})(1+\tau_{\eta})\bigg(M_{1,1}^{\sf T}(\eta)\left(I_{N_{\eta}}\otimes P_{\eta}\right)M_{1,1}(\eta)\nonumber\\
	&+M_{2,1}^{\sf T}(\eta)M_{2,1}(\eta)\bigg) -\left(I_{N_{\eta}}\otimes P_{\eta}\right), \nonumber\\
	\bar M_{2,2}(\eta)&=(1+\rho_{\eta})(1+\frac{1}{\tau_{\eta}})\bigg(M_{1,2}^{\sf T}(\eta)\left(I_{N_{\eta}}\otimes P_{\eta}\right)M_{1,2}(\eta)\nonumber\\
	&\quad+M_{2,2}^{\sf T}(\eta)M_{2,2}(\eta)\bigg)-I,
	%	P_{\eta}&=\sum_{i=0}^{\infty}((A(\eta)-G(\eta)\tilde C(\eta))^i)^{\sf T}(A-G(\eta)\tilde C(\eta))^i.\nonumber
	\end{align}
	where $P_{\eta}$ is defined in \eqref{eq_riccati}.
\end{lemma}
\begin{proof}
	%	See Appendix \ref{pf_lem_parameters}.	
	Consider the claims in \eqref{pf_01}--\eqref{pf_04}. 
	%The conclusion holds if the following  results \eqref{pf_01}--\eqref{pf_04} hold
	%simultaneously,
	\begin{align}
	&(1+\rho_{\eta})(1+\tau_{\eta})M_{1,1}^{\sf T}(\eta)(I_{N_{\eta}}\otimes P_{\eta})M_{1,1}(\eta)\nonumber\\ 
	\preceq& \left(I_{N_{\eta}}\otimes P_{\eta}\right)-\frac{2}{3}I\label{pf_01}\\ 
	&(1+\rho_{\eta})(1+\frac{1}{\tau_{\eta}})M_{1,2}^{\sf T}(\eta)\left(I_{N_{\eta}}\otimes P_{\eta}\right)M_{1,2}(\eta)\nonumber\\
	\preceq& I-\frac{2}{3}I\label{pf_02}
	\end{align}
	\begin{align}
	(1+\rho_{\eta})(1+\tau_{\eta})M_{2,1}^{\sf T}(\eta)M_{2,1}(\eta)&\preceq \frac{1}{3}I\label{pf_03}\\
	(1+\rho_{\eta})(1+\frac{1}{\tau_{\eta}}) M_{2,2}^{\sf T}(\eta)M_{2,2}(\eta)&\preceq \frac{1}{3}I.\label{pf_04}
	\end{align}
	%        \begin{align}\label{pf_5}
	%        \begin{split}
	%      
	%        \end{split}
	%        \end{align}
	If \eqref{pf_01}--\eqref{pf_04} are satisfied simultaneously, then the conclusion holds.
	Thus, we focus on the proofs of \eqref{pf_01}--\eqref{pf_04} in the following.

	1) First, we consider \eqref{pf_01}.
	
	From Lemma~\ref{lem_parameter}, $A-G_{\eta}\tilde C_{\eta}$ is Schur stable, then $P_{\eta}$ in \eqref{eq_riccati} is well defined according to Lemma~\ref{lem_riccati}.  By the Kronecker product property, we have 
	\begin{align}\label{pf_1}
	M_{1,1}^{\sf T}(\eta)\left(I_{N_{\eta}}\otimes P_{\eta}\right)M_{1,1}(\eta)+I=I_{N_{\eta}}\otimes P_{\eta}.
	\end{align} 
	Substitute  \eqref{pf_1} into \eqref{pf_01},  then \eqref{pf_01} is equivalent to 
$	(1+\rho_{\eta})(1+\tau_{\eta})(P_{\eta}-I)\preceq P_{\eta}-\frac{2}{3}I. $
	Since $P_{\eta}$ is a symmetric positive definite matrix, there is a nonsingular matrix $J$ such that 
	$P_{\eta}=J\diag\{\lambda_i(P_{\eta})\}_{i=1}^{n}J^{-1}$, where $\lambda_i(P_{\eta})$ denotes the $i$-th eigenvalue of matrix $P_{\eta}$. Then   \eqref{pf_01} is equivalent to 
$	(1+\rho_{\eta})(1+\tau_{\eta})(\lambda_i(P_{\eta})-1)\preceq \lambda_i(P_{\eta})-\frac{2}{3},$
	for any $ i=1,\dots,n.$
	The above inequalities hold simultaneously from the first condition in \eqref{eq_rho_gamma}  by letting subscript $j=\eta$.
%	Then there exists a scalar $\epsilon_{\eta}>0$, such that the above inequalities hold simultaneously if $\max\{\rho_{\eta},\tau_{\eta}\}<\epsilon_{\eta}.$ 

	2) Second, we consider \eqref{pf_02}.
	
	According to the form of  $M_{1,2}(\eta)$ in \eqref{eq_notation_error1}, $\norm{M_{1,2}(\eta)}$ is sufficiently small provided that $\gamma_{\eta}$ is large enough. We derive that when $L_{\eta}$ is satisfied with the second condition in \eqref{eq_rho_gamma} by letting subscript $j=\eta$,   then \eqref{pf_02} holds.

	3) It remains to prove \eqref{pf_03}--\eqref{pf_04}.
	
	The results in \eqref{pf_03}--\eqref{pf_04}  hold if $\norm{M_{2,1}(\eta)}\leq \left(3(1+\rho_{\eta})(1+\tau_{\eta})\right)^{-1/2}$ and $\norm{M_{2,2}(\eta)}\leq \left(3(1+\rho_{\eta})(1+1/\tau_{\eta})\right)^{-1/2}$.
%	\begin{align}\label{eq_4}
%	\begin{split}
%	\norm{M_{2,1}(\eta)}&\leq \frac{1}{\sqrt{3(1+\rho_{\eta})(1+\tau_{\eta})}}\\
%	\norm{M_{2,2}(\eta)}&\leq \frac{1}{\sqrt{3(1+\rho_{\eta})(1+\frac{1}{\tau_{\eta}})}}.
%	\end{split}
%	\end{align}
	Note that $\mathcal{L}_{\eta}$  is a symmetric matrix, then we derive that 
	the eigenvalues of $I-\alpha_{\eta}(\mathcal{L}_{\eta}\otimes I_n)-\bar I_{\eta}$ is zero and $1-\alpha_{\eta}\lambda_{i}(\mathcal{L}_{\eta}),i=2,\dots,N,$ with each repeated for $n$ times. Thus, it holds that $	\norm{I-\alpha_{\eta}(\mathcal{L}_{\eta}\otimes I_n)-\bar I_{\eta}}_2 
	=\lambda_{c,\eta}.$
	Since the active network is connected, $\lambda_{2}(\mathcal{L}_{\eta})>0$. Then due to $\alpha_{\eta}\in(0,\frac{2}{\lambda_{\max}(\mathcal{L}_{\eta})})$, we have $\lambda_{c,\eta}<1$.
	Since $(I-\bar I_{\eta})\mathbf{G}_{\eta}\mathbf{C}_{\eta}$  belongs to the set orthogonal to consensus set, by Lemma 4.4 in \cite{kar2013distributed}, we have 
$\norm{\left(I-\alpha_{\eta}(\mathcal{L}_{\eta}\otimes I_n)\right)^{L_{\eta}}(I-\bar I_{\eta})\mathbf{G}_{\eta}\mathbf{C}_{\eta}}
\leq \lambda_{c,\eta}^{L_{\eta}}\norm{(I-\bar I_{\eta})\mathbf{G}_{\eta}\mathbf{C}_{\eta}}\leq \lambda_{c,\eta}^{L_{\eta}}\norm{\mathbf{G}_{\eta}\mathbf{C}_{\eta}},$
	where the second inequality holds due to $\norm{I-\bar I_{\eta}}=1.$
	It follows from \eqref{eq_notation_error1} that 
$	\norm{M_{2,1}(\eta)}\leq \gamma_{\eta}\lambda_{c,\eta}^{L_{\eta}}\norm{\mathbf{G}_{\eta}\mathbf{C}_{\eta}}$ and 
$\norm{M_{2,2}(\eta)}\leq \lambda_{c,\eta}^{L_{\eta}}(\norm{A+BK}+\norm{\mathbf{G}_{\eta}\mathbf{C}_{\eta}}),$
	which shows that $\norm{M_{2,1}(\eta)}$ and $\norm{M_{2,2}(\eta)}$ are arbitrarily small provided that $L_{\eta}$ is large enough.  We derive that when $L_{\eta}$ is satisfied with \eqref{L_bound}, 
	the results \eqref{pf_03}--\eqref{pf_04} hold
	simultaneously. 
	
	%From \eqref{pf_0}, \eqref{pf_4}, and \eqref{pf_5}, the conclusion holds.
	
\end{proof}
%The following lemma, proved in \cite{he2020secure}, is useful in the following analysis.
%\begin{lemma}\label{lem_stability}
%	Consider the  linear dynamical system 
%	$	x(t+1)=Fx(t)+G(t),$
%	where $F\in\mathbb{R}^{n\times n}$ is a Schur stable matrix. If $\limsup\limits_{t\rightarrow \infty}\norm{G(t)}\leq \varsigma$,  then $\limsup\limits_{t\rightarrow \infty}\norm{x(t)}\leq \sqrt{\frac{2\theta\varsigma^2\sigma_{\max}(P)}{\sigma_{\min}(P)}},$  where  $P\succ 0$ is the unique solution to $F^{\sf T}PF-P=-I_n$ and 	$\theta=\norm{P}+2\norm{PF}^2$. 	 		
%	%		there is a positive definite matrix $P,$ satisfying $F^{\sf T}PF-P=-I_n$, such that $		\norm{x(t)}^2\leq \frac{\lambda^{\sf T}\sigma_{\max}(P)\norm{x(0)}^2}{\sigma_{\min}(P)}+\frac{\beta}{\sigma_{\min}(P)}\sum_{l=0}^{t-1}\lambda^{t-1-l}\norm{G(l)}^2,$
%	%		%		\begin{align*}%\label{eq_x}
%	%		%		\norm{x(t)}^2\leq& \frac{\lambda^{\sf T}\sigma_{\max}(P)\norm{x(0)}^2}{\sigma_{\min}(P)}+\frac{\beta}{\sigma_{\min}(P)}\sum_{l=0}^{t-1}\lambda^{t-1-l}\norm{G(l)}^2.
%	%		%		\end{align*}
%	%		
%	%		Furthermore, if $\sup_{t\geq 0}\norm{G(t)}\leq \alpha_1$, then $\limsup\limits_{t\rightarrow \infty}\norm{x(t)}^2\leq \frac{2\beta\alpha_1^2\sigma_{\max}(P)}{\sigma_{\min}(P)};$ 
%\end{lemma}
\subsection{Proof of Theorem \ref{thm_error_bound}}\label{pf_thm_bounds}
1) First, we prove the first inequality in \eqref{bound_error}. For task $\eta$ and $t\in[t_{\eta-1},t_{\eta}]$, 
define
the following  function 
$\mathbb{V}(t)=\mathbf{e}(t)^{\sf T}\mathcal{P}_{\eta}\mathbf{e}(t),  $
where $\mathcal{P}_{\eta}=\blockdiag\{I_{N_{\eta}}\otimes P_{\eta},I_{N_{\eta}n}\}.$
%where $P$ is a positive definite matrix such that 
%It can be shown that $P=\sum_{i=0}^{\infty}((A-GC)^i)^{\sf T}(A-GC)^i$ is the solution of \eqref{eq_P}, then  
By  Lemma~\ref{lem_riccati}, $P_{\eta}\succeq I$, which leads to $\lambda_{\max}(\mathcal{P}_{\eta})=\lambda_{\max}(P_{\eta}),
\lambda_{\min}(\mathcal{P}_{\eta})=\lambda_{\min}(P_{\eta}).$
Given $\rho_{\eta}>0$, for $t\in[t_{\eta-1}+1,t_{\eta}]$, it follows from \eqref{lem_extend} that
\begin{align}\label{eq_0}
\mathbb{V}(t)-\mathbb{V}(t-1)\leq \mathbf{e}(t-1)^{\sf T}\hat M(\eta)\mathbf{e}(t-1)+\tilde  M(t)
\end{align}
where 
\begin{align*}
\hat M(\eta)&=(1+\rho_{\eta})\mathbf{M}(\eta)^{\sf T}\mathcal{P}_{\eta}\mathbf{M}(\eta)-\mathcal{P}_{\eta}\nonumber\\
\tilde  M(t)&=(1+1/\rho_{\eta})\mathbf{W}(t)^{\sf T}\mathcal{P}_{\eta}\mathbf{W}(t)\nonumber\\
\mathbf{M}(\eta)&=\begin{pmatrix}
M_{1,1}(\eta)&M_{1,2}(\eta)\\
M_{2,1}(\eta)&M_{2,2}(\eta)
\end{pmatrix}, \mathbf{W}(t)=\begin{pmatrix}
W_{1,1}(t)\\
W_{1,2}(t)
\end{pmatrix}.
\end{align*}

It follows from   the notations in \eqref{eq_notation_error1} that 	$\tilde  M(t)\leq q_{\eta}$, where $q_{\eta}$ is given in \eqref{eq_he}.
According to Lemma \ref{lem_inequality},  
$\hat M(\eta)\preceq \blockdiag\{\bar M_{1,1}(\eta),\bar M_{2,2}(\eta)\},$
where $\bar M_{1,1}(\eta)$ and $\bar M_{2,2}(\eta)$ are given in \eqref{eq_MP}.
%where $\bar M_{1,1}(\eta)$ and $\bar M_{2,2}(t)$ are   in \eqref{eq_notation_error}.
It follows from Lemma \ref{lem_parameters} that   $\bar M_{1,1}(\eta)\preceq -\frac{1}{3}I$ and $\bar M_{2,2}(\eta)\preceq -\frac{1}{3}I$. 
From \eqref{eq_0}, we   have 
$\mathbb{V}(t)\leq \varpi_{\eta} \mathbb{V}(t-1)+q_{\eta}$
where 
$\varpi_{\eta}=1-\frac{1}{3\lambda_{\max}(P_{\eta})}\in(0,1).$
For $\eta\geq 2$ and  $t\in[t_{\eta-1}+1,t_{\eta}]$, it holds that 
\begin{align}\label{eq_V}
\mathbb{V}(t)\leq \varpi_{\eta}^{t-t_{\eta-1}} \mathbb{V}(t_{\eta-1})+q_{\eta}\sum_{l=0}^{t-t_{\eta-1}-1}\varpi_{\eta}^l
\end{align}
For $\eta=1$ and  $t\in[1,t_{1}]$,  it holds that 
\begin{align}\label{eq_V2}
\mathbb{V}(t)\leq \varpi_{1}^{t-1} \mathbb{V}(1)+q_{1}\sum_{l=0}^{t-2}\varpi_{1}^l.
\end{align}
%where 
%\begin{align*}
%\bar q_{\eta}(t)=q_{\eta}\sum_{l=1}^{t-t_{\eta-1}-1}\varpi_{\eta}^l,\quad 
%\hat q_{1}(t)=q_{1}\sum_{l=1}^{t-2}\varpi_{1}^l.
%\end{align*}

%It follows that 
%$\norm{\mathbf{e}(t)}^2\leq \frac{1}{\lambda_{\min}(P)}(\varpi^t \mathbb{V}(0)+d_0\frac{1-\varpi^t}{1-\varpi}).$
Denote $e_{\max}(t,\eta)=\max_{i\in\mathcal{V}_{\eta}}\norm{\hat x_i(t)-x(t)}$ the maximum estimation error of all active sensors for task $\eta$ at time $t$.
According to the form of $\mathbf{e}(t)$, it holds that
\begin{align}
\norm{e_{\max}(t,\eta)}^2&\leq 2\norm{\mathbf{e}_{\net}(t)}^2+2\norm{\mathbf{e}_{\avg}(t)}^2\label{inequl}\\
\norm{e_{\max}(t,\eta)}^2&\geq \max\{\norm{\mathbf{e}_{\avg}(t)}^2/(4N_{\eta}),\norm{\mathbf{e}_{\net}(t)}^2\}.\label{inequl2}
\end{align}
For $t\in[t_{\eta-1}+1,t_{\eta}]$, we derive  
\begin{align}
%&\leq 2\norm{\mathbf{e}_{\net}(t)}^2+2\norm{\mathbf{e}_{\avg}(t)}^2\\
\mathbb{V}(t)&\geq \lambda_{\min}(P(\eta))(N_{\eta}\norm{\mathbf{e}_{\net}(t)}^2+\gamma_{\eta}^2\norm{\mathbf{e}_{\avg}(t)}^2)\nonumber\\
&\geq 2\lambda_{\min}(P(\eta))(\norm{\mathbf{e}_{\net}(t)}^2+\norm{\mathbf{e}_{\avg}(t)}^2)\nonumber\\
&\geq \lambda_{\min}(P(\eta))\norm{e_{\max}(t,\eta)}^2\label{lower_V}
\end{align}
where the second inequality holds due to $N_{\eta}\geq 2$ and $\gamma_{\eta}\geq \sqrt{2}.$
Since the robot shares the latest estimate with the new active sensor network at   task-switching time   $t_{\eta}$,  it holds that $\norm{e_{\max}(t_{\eta-1},\eta)}\leq \norm{e_{\max}(t_{\eta-1},\eta-1)}.$ 
Then according to the notation of $\mathbb{V}(t_{\eta-1})$ and \eqref{inequl2}, we  derive an upper bound of  $\mathbb{V}(t_{\eta-1})$:
\begin{align}\label{eq_Vt}
\mathbb{V}(t_{\eta-1})\leq \lambda_{\max}(P(\eta))(1+4\gamma_{\eta}^2)N_{\eta}\norm{e_{\max}(t_{\eta-1},\eta-1)}^2.
\end{align}
Similarly, we derive an upper bound of  $\mathbb{V}(1)$:
\begin{align}\label{eq_V1}
\mathbb{V}(1)\leq \lambda_{\max}(P(1))(1+4\gamma_{1}^2)N_{1}\norm{e_{\max}(1,1)}^2.
\end{align}

It follows from \eqref{eq_V}--\eqref{lower_V} that: for  $\eta=1$ and $t\in[2,t_{1}]$, 
$\norm{e_{\max}(t,1)}^2\leq  a_{1}(t-1)\norm{e_{\max}(1,1)}^2+b_{1}(t-1),$
and for $\eta\geq 2$ and $t\in[t_{\eta-1}+1,t_{\eta}]$, 
$\norm{e_{\max}(t,\eta)}^2\leq  a_{\eta}(t)\norm{e_{\max}(t_{\eta-1},\eta-1)}^2+b_{\eta}(t),$
where $a_{\eta}(t)$ and $b_{\eta}(t)$ are introduced in \eqref{eq_he}.
Note that $\norm{e_{\max}(1,1)}\leq q_x$ from Assumption~\ref{ass_ini}, then  $\norm{e_{\max}(t,1)}\leq \bar h_e(t)$ for any $t\geq 1$ from the definition of $\bar h_e(t)$ in Algorithm~\ref{alg:bounds}. 

%Then it holds that $\norm{e_{\max}(t_{\eta-1})}\leq \bar h_e(t_{\eta-1})$, where $\bar h_e(t_{\eta-1})$ is defined in \eqref{h_e}. Then 
%from the above inequality, \eqref{eq_V}, and \eqref{lower_V}, $\norm{e_{\max}(t)}\leq \bar h_e(t)$.
%

%we are able to find a sequence $\{h_e(t)\}$ such that $\norm{e_{\max}(t)}\leq h_e(t)$ for all time $t\geq 0$.  For example,  we can  construct $\{h_e(t)\}$  for $t>t_{\eta}$, where  $\eta$ is taken from 1 to $d$,  in a recursive way:
%\begin{align}
%h_e(t)^2&=\frac{\max\{N,\gamma^2\}\lambda_{\max}(P(\eta))}{2\lambda_{\min}(P(\eta))}  \varpi_{\eta}^{t-t_{\eta}} h_e(t_{\eta})^2+d_2(\eta)\nonumber\\
%d_2(\eta)&=\frac{d_1(\eta)}{\lambda_{\min}(P(\eta)}\\
% h_e(t_{\eta})&= h_e(t_{\eta}-1),\nonumber
%%h_e(t_{1})&\geq  \norm{e_{\max}(0)},\nonumber
%\end{align}
%where the third equality is obtained 

%If the system is noise-free with $q_v=0$ and $q_w=0$, then $d_2(\eta)=0$, 
%which ensures 
%$ \lim\limits_{t\rightarrow\infty}h_e(t)=0$ as long as the running time of the last task can tend to infinity.

%\subsection{Proof of Theorem \ref{thm_control}}
%The following lemma, proved by applying Algorithms~\text{\ref{alg:B} and \ref{alg:A}}, shows the dynamics of  $E(t)$ and $x(t)$.
%\begin{lemma}\label{lem_error_state}
%	The following results on the dynamics of $E(t)$ and $x(t)$ hold:
%	\begin{itemize}
%		\item 	The  estimation error $E(t)$ satisfies $E(t+1)=M\bar A(t)E(t)-M(\textbf{1}_N\otimes w(t))W(t)
%		+MGV(t);$ \vskip 5pt
%		\item 
%	\end{itemize} 
%	
%\end{lemma}
2) Next, we prove  the second inequality in \eqref{bound_error}. For $t\in[t_{\eta-1}+1,t_{\eta}],$
  $x(t)$   satisfy the following equation, 
$x(t+1)=(A+BK)x(t)+Bu_{r}(t)+w(t)+BK\big(e_{s_{\eta}}(t)-r(t)\big),$
where    $e_{s_{\eta}}(t)$ is the estimation error of sensor $s_{\eta}$ at time $t$.
%$M=(I_{n\times N}-\alpha(\mathcal{L}_{\eta}\otimes I_n))^L$, $\bar A(t)=I_N\otimes (A+BK)-\textbf{1}_N\otimes ( BK(m_{s_{\eta}}\otimes I_n))-\bar G\bar C$, and  
Then we have 
$x(t+1)-r(t+1) 
=(A+BK)(x(t)-r(t))+G(t),$
where $G(t)=BKe_{s_{\eta}}(t)+w(t).$
Since $A+BK$ is Schur stable, from Lemma~\ref{lem_riccati} $(A+BK)^{\sf T}P_{*}(A+BK)+I_n=P_{*}$ has a unique finite solution $P_{*}\succ 0 $.
By    \cite[Lemma~3]{he2020secure}, $\forall t\in[t_{\eta-1}+1,t_{\eta}],$  it follows that
\begin{align*}
\norm{x(t)-r(t)}^2
\leq& \frac{\lambda^{t-t_{\eta-1}}\lambda_{\max}(P_*)\norm{x(t_{\eta-1})-r_{\eta-1,\eta}(1)}^2}{\lambda_{\min}(P_*)}\nonumber\\
&+\frac{\beta}{\lambda_{\min}(P_*)}\sum_{l=0}^{t-t_{\eta-1}-1}\lambda^{l}\norm{G(t-1-l)}^2.
\end{align*}
Since $\norm{x(t_{\eta-1})-r_{\eta-1,\eta}(1)}
\leq \norm{x(t_{\eta-1})-r(t_{\eta-1})}+\norm{r_{\eta-1,\eta}(1)-r(t_{\eta-1})}
\leq  \bar h_c(t_{\eta-1})+\norm{r_{\eta-1,\eta}(1)-r(t_{\eta-1})},$
   we accomplish the proof by noting that $\norm{G(t)}\leq \norm{BK}\bar h_e(t)+q_w$, $t_{0}=1$, and $\norm{x(1)-r_{0,1}(1)}\leq q_r$ from Assumption~\ref{ass_ini}.

\bibliographystyle{ieeetr}
%\small
\footnotesize
\bibliography{ifacconf}             % bib file to produce the bibliography

\end{document}